\documentclass[lettersize,journal]{IEEEtran}

\usepackage[utf8]{inputenc} 
\usepackage[T1]{fontenc}    
\usepackage{nicefrac}       
\usepackage{microtype}      
\usepackage{lipsum}
\usepackage{amsmath}
\usepackage{amssymb}
\usepackage{mathtools}
\usepackage{amsthm}
\usepackage{amsfonts}       
\usepackage{epsfig}
\usepackage{etoolbox}
\usepackage{booktabs}
\usepackage{graphicx,wrapfig}
\usepackage{multirow}
\usepackage{multicol}
\usepackage{hyperref}
\usepackage{url}            
\usepackage{siunitx}
\usepackage{algorithmic}
\usepackage[style=base]{caption}
\usepackage[font=normalsize,labelfont=sf,textfont=sf]{subcaption}
\usepackage[linesnumbered,lined,boxed,commentsnumbered,ruled,longend]{algorithm2e}
\usepackage{physics}
\usepackage{xcolor}         
\usepackage[capitalize,noabbrev]{cleveref}
\usepackage{array}
\usepackage{textcomp}
\usepackage{stfloats}
\usepackage{verbatim}
\usepackage{cite}
\usepackage{thmtools,thm-restate}

\hypersetup{
   bookmarks, pdftex,
   colorlinks=true,
   pagebackref=true, backref=page,
   linkcolor={red!50!black},
   filecolor={green!50!black},
   citecolor={green!50!black}, 
   urlcolor={blue!80!black},
}

\newcommand{\at}[2][]{#1|_{#2}}

\theoremstyle{plain}
\newtheorem{theorem}{Theorem}

\theoremstyle{definition}
\newtheorem{definition}[theorem]{Definition}

\theoremstyle{remark}

\newcommand{\spara}[1]{\noindent{\bf #1}}

\DeclareMathOperator*{\argmin}{arg\,min}

\newcommand{\shellcmd}[1]{\\\indent\indent\texttt{\footnotesize\# #1}\\}
\usepackage{suffix}
\WithSuffix\def\shellcmd*#1{\indent\indent\texttt{\footnotesize\# #1}\\}

\newcommand{\N}{\mathbb{N}}
\newcommand{\R}{\mathbb{R}}
\newcommand{\C}{\mathbb{C}}
\newcommand{\Z}{\mathbb{Z}}
\newcommand{\I}{\mathbf{I}}

\newcommand{\K}{\mathcal{K}}
\newcommand{\positivefrequency}{\mathcal{F}}
\newcommand{\COVT}{\widehat{\mathsf{C}}}
\newcommand{\PSDT}{\widehat{\mathsf{F}}}
\newcommand{\periodogramT}{\mathring{\mathsf{F}}}
\newcommand{\periodogramM}{\mathring{\mathbf{F}}}
\newcommand{\inPSDT}{\widetilde{\mathsf{F}}}
\newcommand{\inIPSDT}{\widetilde{\mathsf{P}}}
\newcommand{\IPSDT}{\widehat{\mathsf{P}}}
\newcommand{\COVM}{\widehat{\mathbf{C}}}
\newcommand{\PSDM}{\widehat{\mathbf{F}}}
\newcommand{\IPSDM}{\widehat{\mathbf{P}}}
\newcommand{\inPSDM}{\widetilde{\mathbf{F}}}
\newcommand{\inIPSDM}{\widetilde{\mathbf{P}}}
\newcommand{\inRM}{\widetilde{\mathbf{R}}}

\newcommand{\RT}{\widehat{\mathsf{R}}}
\newcommand{\RM}{\widehat{\mathbf{R}}}

\newcommand{\vF}{\underline{\widehat{\mathbf{f}}}}

\newcommand{\vP}{\underline{\widehat{\mathbf{p}}}}
\newcommand{\vI}{\underline{\mathbf{i}}}
\newcommand{\vV}{\underline{\widehat{\mathbf{v}}}}
\newcommand{\vU}{\underline{\widehat{\mathbf{u}}}}
\newcommand{\vX}{\underline{\widehat{\mathbf{x}}}}
\newcommand{\vW}{\underline{\widehat{\mathbf{w}}}}
\newcommand{\vL}{\underline{\widehat{\mathbf{l}}}}
\newcommand{\W}{\widehat{\mathbf{W}}}
\newcommand{\Ll}{\widehat{\mathbf{L}}}
\newcommand{\vZ}{\underline{\mathbf{z}}}
\newcommand{\SELA}{\mathbf{S}_1}
\newcommand{\SEL}{\mathbf{S}_2}
\newcommand{\mO}{\mathbf{O}}
\newcommand{\vO}{\underline{\mathbf{o}}}
\newcommand{\vFtilde}{\underline{\widetilde{\mathbf{f}}}}
\newcommand{\Sp}{\mathcal{S}_{++}^N}

\newcommand{\Se}{\mathcal{S}_{\epsilon}^N}
\newcommand{\vSe}{\underline{\mathcal{S}}_{\epsilon}^N}
\newcommand{\Dtilde}{\widetilde{\mathbf{D}}}
\newcommand{\Dhat}{\widehat{\mathbf{D}}}

\newcommand{\conjugateT}{\mathsf{H}}

\usepackage{balance}

\begin{document}

\title{Learning Multi-Frequency Partial Correlation Graphs\vspace{.3cm}}

\author{Gabriele D'Acunto, Paolo Di Lorenzo~\IEEEmembership{Senior Member,~IEEE}, Francesco Bonchi,\\ Stefania Sardellitti~\IEEEmembership{Senior Member,~IEEE}, Sergio Barbarossa~\IEEEmembership{Fellow,~IEEE} \vspace{-.4cm}
\thanks{Gabriele D'Acunto is with the Department of Computer, Control, and Management Engineering, Sapienza University of Rome, 00185, Italy, and also with Centai Institute, Turin, Italy (e-mail: gabriele.dacunto@uniroma1.it). Francesco Bonchi is with Centai Institute, Turin, Italy, and also with Eurecat (Technological Center of Catalonia), Barcelona, Spain (e-mail: francesco.bonchi@centai.eu). Paolo Di Lorenzo and Sergio Barbarossa are with the Department of Information Engineering, Electronics, and Telecommunications, Sapienza University of Rome, 00184 Rome, Italy (e-mails: paolo.dilorenzo@uniroma1.it; sergio.barbarossa@uniroma1.it).
Stefania Sardellitti is with the Faculty of Engineering in Computer Science, Universitas Mercatorum, University of Italian Chambers of Commerce, 00186 Rome, Italy (e-mail: stefania.sardellitti@unimercatorum.it).
This work was partially supported by the European Union under the Italian National Recovery and Resilience Plan (NRRP)
of NextGenerationEU, partnership on `` Telecommunications of the Future'' (PE00000001 - program `` RESTART'').
}}

\markboth{IEEE TRANSACTIONS ON SIGNAL PROCESSING, 2024}%
{D'Acunto \MakeLowercase{\textit{et al.}}: Learning Multi-Frequency Partial Correlation Graphs}

\maketitle

\begin{abstract}
Despite the large research effort devoted to learning dependencies between time series, the state of the art still faces a major limitation: existing methods learn partial correlations but fail to discriminate across distinct frequency bands.
Motivated by many applications in which this differentiation is pivotal, we overcome this limitation by learning a block-sparse, frequency-dependent, partial correlation graph, in which layers correspond to different frequency bands, and partial correlations can occur over just a few layers.
To this aim, we formulate and solve two nonconvex learning problems: the first has a closed-form solution and is suitable when there is prior knowledge about the number of partial correlations; the second hinges on an iterative solution based on successive convex approximation, and is effective for the general case where no prior knowledge is available.
Numerical results on synthetic data show that the proposed methods outperform the current state of the art.
Finally, the analysis of financial time series confirms that partial correlations exist only within a few frequency bands, underscoring how our methods enable the gaining of valuable insights that would be undetected without discriminating along the frequency domain.
\end{abstract}



\section{Introduction}\label{sec:intro}
\IEEEPARstart{L}{earning} dependencies between time series is a fundamental data-analysis task with widespread application across many domains.
For instance, in \emph{finance}, models for portfolio allocation and risk assessment~\cite{markowitz1952} rely on the study of linear dependencies between the time series of asset returns.
These dependencies may vary not only over time, but also over multiple temporal resolutions, or frequency bands, which have different importance depending on the investment horizon of the portfolio manager.
Similarly, in \emph{neuroscience}, fMRI scanning of the brain measures time series of neural activity in different brain regions of interest (ROIs). A large and growing body of research~\cite{bullmore2009complex} tackles the study of the brain as a network -- or better, a correlation graph -- where an arc between two ROIs (nodes) is created by learning the linear dependencies of the time series associated with the two ROIs.
Also in this case, such dependencies might occur at different temporal resolutions~\cite{salvador2005undirected,jacobs2007brain,ide2017detrended}.
Consider, for example, resting-state fMRI, where 
although the brain is at rest, it exhibits a rich functional activity composed of fluctuations over low-frequency bands in blood oxygen level-dependent signals (see \cite{termenon2016reliability}, and references therein).
When instead the brain is subject to stimuli or performs any tasks, the related functional connectivity occurs at different frequency bands~\cite{ciuciu2014interplay,ide2018time}.
Thus, the ability to capture conditional dependencies occurring at frequency bands relevant to the application context is crucial in retrieving conditional dependencies between brain regions related to different brain states.
The importance of this learning problem is motivated also by other application domains, such as \emph{biology}~\cite{besserve2010causal}, \emph{climatology}~\cite{gu2011precipitation}, and \emph{industrial process monitoring}~\cite{badwe2009detection} (cf. \Cref{subsec:findata,subsec:MPMdata,subsec:Climatedata,subsec:Biodata}).

\spara{Related works.} Learning \emph{linear conditional dependencies} among time series is a well-known problem in statistical learning \cite{songsiri2010graphical,wilson2015models,shrivastava2022methods}. Given a set of time series, the goal is to assess the dependencies between any pair of them, conditioned on the linear effects of others. These dependencies, referred to hereinafter as \emph{partial correlations}, are typically represented through a \emph{partial correlation graph} (PCG), where the time series are the nodes, and the partial correlations are undirected arcs.
Here, the lack of an arc between two nodes indicates that the corresponding time series are linearly statistically independent, conditioned on all the possible instantaneous and lagged linear effects of the other time series.
Partial correlations between time series relate to the zeros of the inverse of the \emph{cross-spectral density} (CSD) matrices, as established in Theorem 2.4 in the seminal work of~\cite{dahlhaus2000graphical}.
The information provided by the inverse CSD for time series is analogous to that provided by the precision matrix for i.i.d. random variables.
This analogy has contributed to the development of techniques that generalize the results obtained for the precision matrix~\cite{banerjee2008model,kolar2010,oh2014optimization,belilovsky2017learning,lugosi2021learning} to the time series context.
Some authors propose shrinkage estimators for the CSD matrices motivated by applications in neuroscience \cite{fiecas2011generalized,fiecas2014data,schneider2016partial}, building upon the shrinkage framework proposed by \cite{bohm2009shrinkage} for data-driven $\ell_2$-penalised estimation of the spectral density matrix. Another stream of research focuses on learning the PCG under sparsity constraints. Specifically, the paper in \cite{fiecas2019spectral} studies a fixed-sample high-dimensional setting, and proposes a method related to the method of constrained $\ell_1$-minimization for inverse matrix estimation (CLIME,~\cite{cai2011constrained}). Other approaches~\cite{jung2015graphical,foti2016sparse,dallakyan2022time,tugnait2022sparse,deb2024regularized} leverage the \emph{Whittle approximation} (WA,~\cite{whittle1953analysis,whittle1957curve}) for stationary processes. 
More recently,~\cite{krampe2024frequency} proposes a nonparametric estimator for the \emph{partial spectral coherence} (which depends on the inverse CSD, cf. \Cref{sec:background}), providing also a procedure to assess whether the latter exceeds a user-specified threshold on a frequency band of interest. 

Regarding these previous works, we highlight two major limitations. Firstly, existing methodologies focus on learning PCGs without the possibility of discriminating across different frequency bands, while this is important in many applications (as discussed above). 
Secondly, existing methods take estimated CSD matrices as input and keep them unaltered during learning. Thus, the goodness of the solution heavily depends on the accuracy of the CSD matrices estimation.

\spara{Contributions.}  This paper proposes methods to learn partial correlations between time series across multiple frequency bands, overcoming the three major limitations of the state of the art described above. For what concerns the first limitation, we propose the learning of a \emph{frequency-dependent PCG}, where different layers correspond to different frequency bands, and where partial correlations can possibly occur only over some frequency bands.
To overcome the second limitation, we introduce a novel methodology that jointly learns the CSD matrices and their inverses, without specifically hinging on some predefined CSD estimators.
To be more specific, our proposal comprises two methods based on different assumptions. For the case where prior domain knowledge about the number of partial correlations between time series is available, we expose a problem formulation (\Cref{sec:nonconvex_close}) that has a closed-form solution and builds upon theoretical results in the compressed sensing literature~\cite{baraniuk2010model}. For the general case where such prior knowledge is not available, we formulate an optimization problem (\Cref{sec:nonconvex}) to jointly learn the CSD matrices and their inverses, and we devise an iterative solution of this problem based on successive convex approximation methods (SCA,~\cite{scutari2018parallel}). In the rest of the paper, for simplicity, we dub the former method CF (for ``closed-form'') and the latter method IA (for ``iterative approximation''). Our methods have broad applicability as they are not limited by any particular statistical model, enabling a more general approach for learning partial correlations between time series across different frequency bands. Our experimental assessment on synthetic data (\Cref{sec:synth_exp}) demonstrates the superiority of our proposals over the baselines. We also present a real-world case study in the financial domain (\Cref{sec:rwdata}), confirming that in real-world time series, partial correlations might either concentrate at a certain frequency band or spread across multiple frequencies.
This highlights the importance of learning block-sparse multi-frequency PCGs and shows that our methodology can help in extracting richer information from data.
Overall, our proposals offer an important contribution to the field of time series analysis and have potential for applications in various domains.
The JAX~\cite{jax2018github} implementation of our algorithms, along with the material for results reproduction is accessible \href{https://github.com/OfficiallyDAC/BSPCG}{here}.

\noindent\spara{Roadmap.}
The paper is organized as follows.
\Cref{sec:background} introduces key concepts while \Cref{sec:preliminaries} the spectral properties of interest and their naive estimators. Next, \Cref{sec:probform} describes the learning problem, tackled through the formulation and solution of two nonconvex problems, detailed in \Cref{sec:nonconvex_close,sec:nonconvex}, respectively.
Then, \Cref{sec:synth_exp} addresses the empirical assessment of the proposed algorithms on synthetic data, while \Cref{sec:rwdata} outlines practical scenarios in which our methods are suitable,
but not limited to, and showcases the application on financial time series.
Finally, \Cref{sec:conclusion} draws the conclusions.

\spara{Notation.}
Scalars are lowercase, $a$, vectors are lowercase bold, $\mathbf{a}$, matrices are uppercase bold, $\mathbf{A}$, and tensors are uppercase sans serif, $\mathsf{A}$. The set of integers from $1$ to $N$ is denoted by $[N]$, and $[N]_0$ if zero is included. The floor of $T \in \R$ is $\lfloor T \rfloor$.
The imaginary unit is $\iota$.
The conjugate of $\mathbf{a}$ is $\mathbf{a}^*$, and the conjugate transpose is $\mathbf{a}^\conjugateT$ (the same for matrices).
The complex signum is $\mathrm{sign}(\mathbf{a})=e^{\iota\,\arg \mathbf{a}}$.
$\I_M$ is the identity matrix of size $M \!\times\! M$, $\mathbf{J}_M$ is the $M \!\times\! M$ matrix of ones, $\mathbf{J}_M^L$ is the $M \!\times\! M$ strictly lower triangular matrix of ones. 
The entry indexed by row $i$ and column $j$ is $a_{ij}=[\mathbf{A}]_{ij}$, $\mathrm{diag}(\mathbf{a})$ is the diagonal matrix having as diagonal the vector $\mathbf{a}$, and $\underline{\mathbf{a}}\!=\!\mathrm{vec}(\mathbf{A}), \, \mathbf{A} \in \C^{C \times D}$, denotes the vectorization of $\mathbf{A}$ formed by stacking the columns of $\mathbf{A}$ into a single column vector. The set of positive definite matrices over $\C^{N\times N}$ is denoted by $\mathcal{S}^N_{++}$.
The mixed $(p,q)$ norm is $\norm{\mathbf{A}}_{p,q} \!=\! \left(\sum_{d \in [D]} \norm{\mathbf{a}_d}_p^q \right)^{1/q}$, where $\{\mathbf{a}_d\}_{d\in [D]} \in \C^C$ are the columns of $\mathbf{A}$. Additionally, the $(p,0)$ norm is the number of nonzero columns of $\mathbf{A}$, the $(2,2)$ norm is the Frobenious norm $\norm{\mathbf{A}}_F$, $\norm{\mathbf{A}}_{2}$ is the largest singular value of $\mathbf{A}$, and $\norm{\mathbf{A}}_{\infty}=\max_{ij} \abs{a_{ij}}$ is the element-wise infinity norm. Given a $3$-way tensor $\mathsf{A}$ indexed by $i$, $j$, and $k$, the fiber $[\mathsf{A}]_{ij:}$ is denoted as $[\mathsf{A}]_{ij}$. The Kronecker product is $\otimes$.
Finally, given a closed nonempty convex set $\mathcal{C}$, the set indicator function $I_{\mathcal{C}}(x)$ is defined as
\begin{equation*}
    I_{\mathcal{C}}(x)=\begin{cases}
        0 & \text{if $x \in \mathcal{C}$}, \\
        +\infty & \text{otherwise.}
    \end{cases}
\end{equation*}

\section{Conditional linear independence over frequency bands}\label{sec:background}
In this section, we briefly recall the main results from \cite{dahlhaus2000graphical} about the partial correlation between multivariate time series, expressed in the frequency domain, as they will be used later on in this paper.
Consider an $N$-variate, zero-mean, weakly-stationary process $\mathbf{y}[t]=[y_{1}[t], \ldots, y_{N}[t]]^{\top} \in \R^{N}, \, t \in \mathbb{Z}$.
The autocovariance function reads as the following matrix-valued function
\begin{equation}\label{eq:autocov}
    \mathbf{C}_l=\mathbb{E} [\mathbf{y}[t+l] \mathbf{y}[t]^{\top}],
\end{equation}
with $l \in \mathbb{Z}$.
Assuming that $\sum_{l \in \mathbb{Z}}\norm{\mathbf{C}_l}_2<\infty$, the cross-spectral density function is the following matrix-valued function over $\R$:
\begin{equation}\label{eq:spectraldensity}
    \mathbf{F}_{\nu} = \frac{1}{2\pi}\sum_{l\in \mathbb{Z}} \mathbf{C}_l e^{-\iota 2 \pi \nu l},
\end{equation}
with $\nu \in [0,1]$; while we denote its inverse as $\mathbf{P}_{\nu}$.
Since $\mathbf{F}_{\nu}$ is Hermitian for all $\nu$, also $\mathbf{P}_{\nu}$ is Hermitian; furthermore, $\mathbf{F}_{\nu}=\mathbf{F}^\conjugateT_{-\nu}$ since $\mathbf{y}[t]$ is real-valued. Thus, we can focus only on $\nu \in [0, 0.5]$.
According to Theorem 2.4 in \cite{dahlhaus2000graphical}, rescaling $\mathbf{P}_{\nu}$ leads to the \emph{partial spectral coherence}, $\mathbf{R}_{\nu}=-\mathbf{D}_{\nu} \mathbf{P}_{\nu}\mathbf{D}_{\nu}$, where $\mathbf{D}_{\nu}$ is a diagonal matrix with entries $[\mathbf{P}_{\nu}]_{ii}^{\frac{1}{2}}, \, \forall i \in [N]$.
Let us consider, w.l.o.g., $y_{1}[t]$ and $y_{2}[t]$.
Define $\mathbf{y}_{12}[t]\coloneqq[y_{3}[t], \ldots, y_{N}[t]]^{\top} \in \R^{N-2}$, and consider the best linear prediction of $y_{1}[t]$ and $y_{2}[t]$ in terms of the time series in $\mathbf{y}_{12}[t]$, i.e.,
\begin{align}
    y_{1}[t] &=\! \sum_{l=-\infty}^{\infty} \mathbf{d}_1[l]^\top \mathbf{y}_{12}[t-l] \!+\! \epsilon_1[t]\!=\!(\mathbf{d}_1 \!\star\! \mathbf{y}_{12}) + \epsilon_1[t], \label{eq:firstapprox}\\
    y_2[t] &=\! \sum_{l=-\infty}^{\infty} \mathbf{d}_2[l]^\top \mathbf{y}_{12}[t-l] \!+\! \epsilon_2[t]\!=\!(\mathbf{d}_2 \!\star\! \mathbf{y}_{12}) + \epsilon_2[t];\label{eq:secondapprox}
\end{align}
where $\mathbf{d}_1[l], \, \mathbf{d}_2[l] \in \R^{N-2}$; $\epsilon_1[t], \, \epsilon_2[t] \in \R$ are possible models mismatch; and $\star$ is the convolution operation.
Now, let us denote with $\mathcal{F}_{\nu}\{\cdot\}$ the Fourier transform at frequency $\nu$.
From \eqref{eq:firstapprox}, applying the Fourier transform, we obtain:
\begin{equation}
    \begin{aligned}
        \tilde{\epsilon}_1[\nu]&=\mathcal{F}_{\nu}(\epsilon_1[t])=\mathcal{F}_{\nu}\{y_{1}[t] - (\mathbf{d}_1 \star \mathbf{y}_{12})\}=\\
        &\stackrel{(a)}= \mathcal{F}_{\nu}\{y_{1}[t]\} - \mathcal{F}_{\nu}\{(\mathbf{d}_1 \star \mathbf{y}_{12})\}=\\
        &\stackrel{(b)}= \tilde{y}_1[\nu]-\tilde{\mathbf{d}}_1[\nu]\tilde{\mathbf{y}}_{12}[\nu];
    \end{aligned}
\end{equation}
where we exploit $(a)$ the linearity of Fourier transform, and $(b)$ the convolution theorem.
Similarly, from \eqref{eq:secondapprox} we obtain:
\begin{equation}
    \tilde{\epsilon}_2[\nu] = \tilde{y}_2[\nu]-\tilde{\mathbf{d}}_2[\nu]\tilde{\mathbf{y}}_{12}[\nu].
\end{equation}
The measure of linear dependence between $\tilde{\epsilon}_1[\nu]$ and $\tilde{\epsilon}_2[\nu]$ is $f_{\tilde{\epsilon}_1\tilde{\epsilon}_2}[\nu]=\mathbb{E}[ \tilde{\epsilon}_1[\nu] \tilde{\epsilon}_2^{*}[\nu] ]$, which coincides with the partial cross-spectrum \cite{dahlhaus2000graphical}.
From eq. (2.2) in \cite{dahlhaus2000graphical}, we have that $f_{\tilde{\epsilon}_1\tilde{\epsilon}_2}[\nu]=0$ iff $[\mathbf{R}_{\nu}]_{12}=0$.
Given the definition of $\mathbf{R}_{\nu}$, $[\mathbf{P}_{\nu}]_{12}=0$ implies no correlation between $y_1[t]$ and $y_2[t]$ once they have been bandpass filtered at frequency $\nu$, after removing the linear effects of $\mathbf{y}_{12}[t]$ (for all lags $l \in \mathbb{Z}$).
Consequently, if $\{[\mathbf{P}_{\nu}]_{12}\!=\!0\}_{\nu=a}^{b}$, the \emph{bandpass-filtered representations} of $y_1[t]$ and $y_2[t]$ over the frequency band $[a,b]$ are linearly independent conditionally on $\mathbf{y}_{12}[t]$.
In practice, we deal with the finite-sample setting in which one is interested in identifying frequency bands in which the partial coherence is greater than in others where it is negligible.
These key concepts lead to \Cref{def:KPCG} in \Cref{sec:probform}.

\section{Estimation of Inverse CSD Tensor}\label{sec:preliminaries}
In this section, we introduce the basic tools for the estimation of the CSD tensor and its inverse from time series data, defined below.
Consider a set of time series, denoted as $\mathbf{Y} \in \R^{N \times T}$, where $T$ represents the number of samples in the data set and $N$ represents the number of time series. The data set comprises samples of an $N$-variate, zero-mean, and weakly stationary process, denoted as $\mathbf{y}[t] \!:= \![y_1[t], \ldots, y_N[t]]^\top$ for $t \in [T]$. Since in the finite sample setting, we denote the CSD matrix at rescaled frequency $\nu_k\!=\!k/T, \, k \! \in \! \positivefrequency, \, \positivefrequency \coloneqq \{0,\ldots,\lfloor T/2 \rfloor\}$, as $\mathbf{F}_k \!\in\! \C^{N\times N}$, and its inverse as $\mathbf{P}_k$, which we assume to exist. 
Similarly to \Cref{sec:background}, $\mathbf{F}_k$ and $\mathbf{P}_k$ are Hermitian and symmetric; thus we focus only on positive frequencies $k \! \in \! \positivefrequency$.

Let us now introduce the \emph{CSD tensor} $\mathsf{F}=\{\mathbf{F}_k\}_{k\in \positivefrequency}\in\C^{N\times N \times M}$ consisting of $M=|\mathcal{F}|=\lfloor T/2 \rfloor +1$ slices of size $N$ by $N$, where each slice represents the CSD matrix corresponding to a certain frequency.
To streamline notation, hereinafter we omit the subscript $k\in \positivefrequency$ in collections indexed by $k$.
An estimator of the CSD tensor is the periodogram, $\periodogramT \! \in \! \C^{N\times N \times M}$, defined as the discrete Fourier transform (DFT) of the sample autocovariance $\COVT$, defined as the collection of the matrices \[\COVM_l=\dfrac{1}{T} \sum_{t=0}^{T-|l|-1}(\mathbf{y}[t+l]-\bar{\mathbf{y}})(\mathbf{y}[t]-\bar{\mathbf{y}})^\top,\] for $l \in [-T+1, T-1]$, where $\bar{\mathbf{y}}=1/T \sum_{t=0}^{T-1} \mathbf{y}[t]$. Specifically, $\periodogramT=\{\periodogramM_k\}$ is the collection of the matrices $\periodogramM_k$, having entries \[\left[\periodogramM_k\right]_{ij}=\sum_{l=-T+1}^{T-1} \left[\COVM_{l}\right]_{ij} e^{-\iota 2 \pi k l /T}.\]

It is well known that the periodogram is not a consistent estimator of the CSD tensor \cite{fiecas2019spectral}. A common remedy is to smooth (i.e., average) it, also in high-dimensional fixed-sample settings (Theorem 3.1 in~\cite{fiecas2019spectral}). 
The periodogram is smoothed by convolving it with a smoothing window having smoothing span $2\varsigma+1$. The window is required to be symmetric and to have nonnegative weights summing up to one.
Typically, the smoothing window reaches its maximum at the center, and decays smoothly and fast enough as we move from the latter.  
There exist several windows satisfying the previous conditions, however, the choice of the window shape is secondary to that of the half-window size $\varsigma$.
Indeed, for making the smoothing periodogram consistent, $\varsigma\!=\!\varsigma(t)$ must satisfy (i) $\varsigma(t)\!\rightarrow \!\infty$ and (ii) $\varsigma(t)/T \!\rightarrow \!0$ as $T \!\rightarrow\! \infty$.
We denote the smoothed periodogram by $\inPSDT$.
Then, we obtain the \emph{naive} estimator of the inverse CSD tensor, $\inIPSDT=\{\inIPSDM_{k}\}$, such that $\inPSDM_{k} \inIPSDM_{k} \!=\! \I_{N}$.
Finally, we have the \emph{partial spectral coherence} estimator $\inRM_{k}\!=\!-\Dtilde_{k}\inIPSDM_{k}\Dtilde_{k}$, where $\Dtilde_{k}$ is a diagonal matrix with entries $[\inIPSDM_{k}]_{ii}^{-\frac{1}{2}}$,  $\forall i \in [N]$ ~\cite{dahlhaus2000graphical}. 
\section{Problem statement}\label{sec:probform}

\begin{figure}[t]
    \centering
    \includegraphics[width=1\columnwidth]{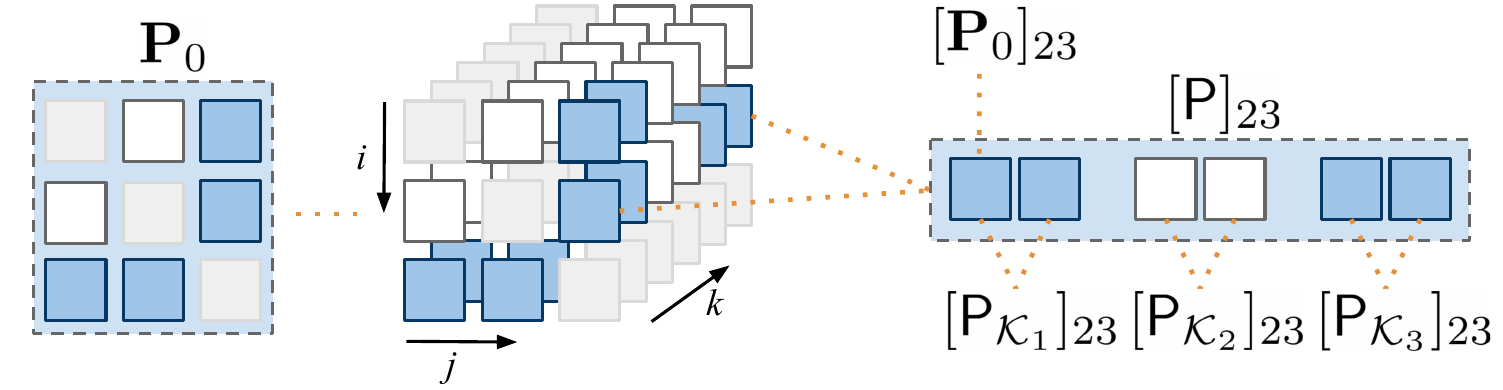}
    \caption{Inverse CSD tensor and its components. Nonzero off-diagonal entries are given in blue, white otherwise. Diagonal entries are shaded. The block-sparsity feature refers to the occurrence of partial correlations only over some $\mathcal{K}_m$. In this example, time series $i=2$ and $j=3$ are not partially correlated at the second block of frequencies $\mathcal{K}_2$, where $m \in [3]$ and each block is made of two frequencies.}
    \label{fig:probstat}
\end{figure}

In this section, we state our learning problem, introducing the definition of multi-frequency partial correlation graph, where each layer refers to a different frequency band. Consider as input the data set $\mathbf{Y} \in \R^{N \times T}$ mentioned above, and let us focus on the frequencies $k \! \in \! \positivefrequency$.
We partition the frequency range into $K$ consecutive blocks of interest denoted as $\K_m$, where $m \in [K]$ and $0\!<\!K\!<\!M$.
The starting frequency index of block $\K_m$ is denoted by $k_m$, and $\K_m$ corresponds to the frequency band $[\nu_{k_m}, \nu_{k_{m+1}})$.
We refer to the tensor collecting the inverse CSD matrices over $\K_m$ as $\mathsf{P}_{\K_m}$. Clearly, the inverse CSD tensor $\mathsf{P}\in\mathbb{C}^{N\times N\times M}$ is equivalently given by $\mathsf{P}=\{\mathbf{P}_{k}\}=\{\mathsf{P}_{\K_m}\}_{m \in [K]}$. As a pictorial example, \Cref{fig:probstat} depicts the inverse CSD tensor $\mathsf{P}$ and its components.
In light of \Cref{sec:background}, we introduce the multi-frequency partial correlation graph made of $K$ independent layers.

\begin{definition}[$K$-frequency partial correlation graph]\label{def:KPCG}
The $K$-frequency partial correlation graph ($K$-PCG) is a graph composed of $K$ independent layers, where the $m$-th layer  is an undirected graph $\mathcal{G}_m(\mathcal{V},\mathcal{E}_m)$ associated to the frequency band $\K_m$, such that $\mathcal{V}=[N]$, and $\mathcal{E}_m=\{e_{ij} \,|\, \left[\mathsf{P}_{\K_m}\right]_{i j}\neq 0 \, \, \text{for some } \nu_k \in \K_m, \ (i,j)\in [N]\times[N], i\neq j\}$.
\end{definition}

According to \Cref{def:KPCG}, the presence of an arc $e_{ij} \in \mathcal{E}_m$ depends on the values of $\left[\mathsf{P}_{\K_m}\right]_{ij}$ within $\K_m$, which must be different from zero over at least one frequency component. Thus, if the arc $e_{ij}$ is absent for every recovered graph $\mathcal{G}_m$, time series $\mathbf{y}_i$ and $\mathbf{y}_j$ are not partially correlated~\cite{dahlhaus2000graphical}. 
We remark that the absence of the arc $e_{ij}$ in a layer $\mathcal{G}_m$ means that the \emph{bandpass-filtered representations} of the time series $\mathbf{y}_i$ and $\mathbf{y}_j$ at the frequency band corresponding to the layer are not partially correlated.
We say that a $K$-PCG is block-sparse if partial correlations (i.e., arcs) exist only over some frequency bands $\K_m$. Then, driven by \Cref{def:KPCG} and the above considerations, in this work we aim to find a block-sparse graph representation associated with each frequency band in a data-driven manner, without any assumptions about the underlying statistical model. We aim to develop an approach that is robust to possible numerical fluctuations and is applicable even when $N>T$. In the sequel, we formulate this problem mathematically according to two different criteria (cf. \Cref{sec:nonconvex_close,sec:nonconvex}, respectively), imposing block-sparsity in the estimate of the inverse CSD tensor along the frequency domain.
To ease comprehension, \Cref{tab:objects} provides the main concepts, along with their description and corresponding symbols throughout the text.

\begin{table*}[t]
    \centering
    \caption{Main concepts along with their description and corresponding symbols throughout the text.}
    \resizebox{1.\textwidth}{!}{%
    \begin{tabular}{c p{0.5\textwidth} p{0.3\textwidth}}
        \toprule
        Concept& Description & Symbol \\
        \midrule
        Partial correlation graph & Graphical representation of linear conditional dependencies between multivariate time series. Nodes correspond to the time series and undirected edges to linear conditional dependencies.& PCG.\\
        \midrule
        $K$-frequency partial correlation graph& Graphical representation of linear conditional dependencies between multivariate time series at $K$ distinct frequency bands. Each band corresponds to a layer of the graph. Within each layer, nodes correspond to the bandpass-filtered time series and undirected edges to linear conditional dependencies over the associated frequency band.&$K$-PCG.\\
        \midrule
        Autocovariance& It measures the variance in and the covariance between the dimensions of a multivariate time series system, at different lags $l \in \Z$.& $\mathsf{C}=\{\mathbf{C}_l\}_{l \in \Z}$ (tensor), $\mathbf{C}_l$ (matrix), $\COVT=\{ \COVM_l\}_{l \in \Z}$ (tensor estimate), $\COVM_l$ (matrix estimate).\\
        \midrule
        Cross-spectral density & Frequency-domain analog of the covariance matrix. It characterizes the variance in and the covariance between the dimensions of a multivariate time series system attributed to oscillations in the data at different frequencies $\nu \!\in\! [0,1]$.& $\mathsf{F}=\{\mathbf{F}_k\}$ (tensor), $\mathbf{F}_k$ (matrix), $\PSDT=\{\PSDM_k\}$ (tensor estimate), $\PSDM_k$ (matrix estimate), $\vF_k=\mathrm{vec}(\PSDM_k)$.\\
        \midrule
        Inverse cross-spectral density & Inverse of the cross-spectral density matrix, used to estimate the partial spectral coherence. It is the frequency-domain analog of the precision matrix. It quantifies linear conditional dependencies attributed to variations in the oscillations of the data at different frequencies $\nu\! \in \![0,1]$.& $\mathsf{P}=\{ \mathsf{P}_{\K_m}\}_{m \in K}=\{\mathbf{P}_k\}$ (tensor), $\mathsf{P}_{\K_m}$ (tensor slices corresponding to $\K_m$), $\mathbf{P}_k$ (matrix), $\mathbf{P}_{(\K_m)}$ (tensor flattening along $\K_m)$, $\IPSDT=\{\IPSDM_k\}$ (tensor estimate), $\IPSDM_k$ (matrix estimate), $\IPSDM_{(\K_m)}$ (tensor estimate flattening along $\K_m)$,  $\vP_k=\mathrm{vec}(\IPSDM_k)$, $\vP_{(\K_m)} = \mathrm{vec}( \IPSDM_{(\K_m)})$.\\
        \midrule
        Partial spectral coherence & Rescaled version of the inverse cross-spectral density. Frequency-domain analog of the partial correlation. It measures the correlation between two time series that have been bandpass filtered at frequency $\nu$, after accounting for the linear effects of other time series in the system.& $\mathsf{R}=\{ \mathsf{R}_{\K_m}\}_{m \in K}=\{ \mathbf{R}_k \}$ (tensor), $\mathsf{R}_{\K_m}$ (tensor slices corresponding to $\K_m$),$\mathbf{R}_k$ (matrix), $\RT=\{ \RT_{\K_m}\}_{m \in K}= \{ \RM_k\}$ (tensor estimate), $\RM_k$ (matrix estimate).\\
        \bottomrule
    \end{tabular}}
    \label{tab:objects}
\end{table*}

\section{The closed-form (CF) method}\label{sec:nonconvex_close}
Let us assume that the true inverse CSD tensor $\mathsf{P}$ is block-sparse over $K$ distinct blocks. Further, we consider the number of unique nonzero off-diagonal fibers (mode-$k$) within each frequency band $\K_m$ equal to $2s_m$ for all $ m \in [K]$, with $s_m \in \N$. Note that, since the true inverse CSD is Hermitian, the number of nonzero off-diagonal fibers $[\mathsf{P}]_{ij}, \, (i, j) \in [N]\times[N],\, i\neq j,$ can only be even. As per \Cref{def:KPCG}, the tensor $\mathsf{P}$ entails a $K$-PCG that has arc sets $\mathcal{E}_m$ with cardinality $s_m$, $\forall m \in [K]$.

The rationale behind our proposal is to interpret the naive estimator $\inIPSDT$ (cf. \Cref{sec:preliminaries}) as a noisy measurement of the true block-sparse inverse CSD tensor $\mathsf{P}$.
As a consequence, if we further assume to have prior knowledge about the sparsity values $s_m$, $\forall m \in [K]$, we can cast the learning of the $K$-PCG as a \emph{block-based signal recovery} problem~\cite{baraniuk2010model}. In particular, let us focus on a specific frequency band $\K_m$. We denote the flattening (a.k.a. matricization) of the true inverse CSD tensor $\mathsf{P}_{\K_m}$, along the frequency interval $\K_m$, as $\mathbf{P}_{(\K_m)} \in \C^{|\K_m|\times N^2}$. Here, each column of the matrix $\mathbf{P}_{(\K_m)}$, i.e., $\mathbf{p}_{i+N(j-1)}$ with $(i, j) \in [N]\times[N]$, contains the entries $[\mathbf{P}_k]_{ij}$ of the inverse CSD matrices $\forall \nu_k \in \K_m$. Thus, the block-sparsity assumption on the $K$-PCG coincides with having only $2s_m$ columns of $\mathbf{P}_{(\K_m)}$ to be not entirely equal to zero (cf. \Cref{def:KPCG}). 
In addition, since the true inverse CSD is Hermitian, the columns of $\mathbf{P}_{(\K_m)}$ satisfy $\mathbf{p}_{i+N(j-1)}=\mathbf{p}_{j+N(i-1)}^*$, which means that half of the information in $\mathbf{P}_{(\K_m)}$ is redundant and can be neglected in our formulation. To this end, let us consider the selection matrix $\SELA=\mathrm{diag}(\underline{\mathbf{j}}_{N}^{L})$, where $\underline{\mathbf{j}}_{N}^{L}$ is the vectorization of $\mathbf{J}_N^L$. The product $\mathbf{P}_{(\K_m)} \SELA$ retains only the columns of $\mathbf{P}_{(\K_m)}$ corresponding to the strictly lower diagonal fibers of $\mathsf{P}_{\K_m}$, which contain all the information needed to learn the arc set $\mathcal{E}_m$ with cardinality $s_m$. 

\begin{algorithm}[t]
    \small
    \caption{CF method.}
    \label{alg:convex}
    \begin{algorithmic}[1]
       \STATE {\bfseries Input:} $\{\inIPSDM_{(\K_m)}\}_{m\in[K]}$, $\mathbf{s}=[s_1,\ldots,s_K]^\top \in \N^K$.
       \STATE {\bfseries Output:} $\{\IPSDM_{(\K_m)}\}_{m \in [K]}$.
       \STATE $\SELA \gets \mathrm{diag}(\mathrm{vec}(\mathbf{J}_{N}^L))$, $\bar{\mathbf{S}}_1 \gets \mathrm{diag}(\mathrm{vec}(\mathbf{J}_{N}^{L^\top}))$
       \STATE Initialize $\IPSDM_{(\K_m)} \gets \inIPSDM_{(\K_m)}\mathrm{diag}(\vI_N), \, \forall m \in [K]$
       \FOR{$m=1$ {\bfseries to} $K$}
       \STATE $\IPSDM_{(\K_m)}\SELA \gets  \mathrm{top-}s_m\left(\norm{\inIPSDM_{(\K_m)}\SELA}_2\right)$
       \STATE $\IPSDM_{(\K_m)}\bar{\mathbf{S}}_1 \gets \left(\IPSDM_{(\K_m)}\SELA\right)^*$
       \ENDFOR
    \end{algorithmic}
\end{algorithm}
Finally, let $\inIPSDM_{(\K_m)} \in \C^{|\K_m|\times N^2}$ be the flattening of the naive estimator $\inIPSDT_{\K_m}$. Then, we can cast our learning problem as the search for the $s_m$-block-sparse matrix $\IPSDM_{(\K_m)}\SELA$ that approximates $\inIPSDM_{(\K_m)}\SELA$, for each $\K_m, \, m \in [K]$,
\begin{equation}
    \begin{aligned}
    \min_{\IPSDM_{(\K_m)}} \quad & ||\inIPSDM_{(\K_m)}\SELA - \IPSDM_{(\K_m)}\SELA||_F^2 \\
    \textrm{subject to} \quad & ||\IPSDM_{(\K_m)}\SELA||_{2,0} \leq s_m\,.\\
    \end{aligned}
    \tag{P1}\label{eq:prob0}
\end{equation}
Problem \eqref{eq:prob0} is nonconvex due to the block-sparsity inducing constraint on the $\ell_{2,0}$ norm. Nevertheless, following the results in~\cite{baraniuk2010model,sardellitti2019graph}, \eqref{eq:prob0} admits a closed-form (globally optimal) solution. Specifically, for each $\K_m$, the best $s_m$-block-sparse approximation $\IPSDM_{(\K_m)}$ can be recovered by simply sorting the columns of $\inIPSDM_{(\K_m)}\SELA$ according to their $\ell_2$-norm and then retaining the top-$s_m$ columns. The $s_m$ selected columns are associated with the lower-diagonal fibers of $\mathsf{P}_{\K_m}$, and are sufficient to identify the arc set $\mathcal{E}_m$ with cardinality $s_m$. Finally, if needed, the remaining structure of $\IPSDM_{(\K_m)}$, i.e., the one associated with diagonal and upper-diagonal fibers of $\mathsf{P}_{\K_m}$, can be simply obtained copying the diagonal of $\inIPSDM_{(\K_m)}$ and taking the conjugate of $\IPSDM_{(\K_m)}\mathbf{S}_1$, respectively.  All the steps of the proposed method, named CF, are given in~\Cref{alg:convex}.\\
\textbf{Computational cost.} For each block of frequencies $\K_m$, we compute the $\ell_2$-norm of $N(N-1)/2$ terms, each consisting of $|\K_m|$ multiplications and $|\K_m|-1$ additions. Therefore, these operations require $\mathcal{O}(N^2|\K_m|)$ flops.  Afterward, we choose the top-$s_m$ largest terms, with a cost $\mathcal{O}(N^2 \log s_m)$. 
Hence, the cost for each block is $\mathcal{O}(N^2(|\K_m|+\log s_m)),\, m\in [K]$.

\section{The iterative approximation (IA) method}\label{sec:nonconvex}
The CF method in \Cref{sec:nonconvex_close} hinges on prior knowledge about the sparsity measure $s_m, \, \forall m \in [K]$, which is hardly available in practical scenarios. In this section, we propose a method that does not exploit such prior information. In addition, differently from the CF method in \Cref{sec:nonconvex_close}, here we do not consider the CSD tensor estimate as fixed during the learning process, but we jointly learn the CSD tensor and its inverse, which we denote as $\{\PSDM_k\}\in \mathcal{S}^N_{++}$ and $\{\IPSDM_k\}\in \mathcal{S}^N_{++}$, respectively. To this aim, since $\mathcal{S}^N_{++}$ is an open set that is difficult to handle with iterative optimization methods, we first introduce an approximation of $\mathcal{S}^N_{++}$ given by the closed set 
\begin{equation}\label{eq:Se}
    \Se=\{\mathbf{A}\in \Sp \,|\, \mathbf{a}^\conjugateT\mathbf{A}\mathbf{a}\geq\epsilon, \, \epsilon \in \R_+, \, \forall\mathbf{a} \in \C^N\}\, ,
\end{equation}
where $\epsilon>0$ can be chosen sufficiently small to well approximate $\mathcal{S}^N_{++}$. Also, let  $\vSe$ be the closed set of vectorized matrices in \eqref{eq:Se}. In second place, hinging on \Cref{def:KPCG}, we notice that only the cross-interaction fibers $[\IPSDT_{\K_m}]_{ij}, \, i \neq j$, might yield arcs in the corresponding layer of the learned $K$-PCG. Then, similarly to (\ref{eq:prob0}), we introduce the selection matrix $\SEL=\mathrm{diag}(\vO)$, $\SEL \in \{0,1\}^{N^2 \times N^2}$, being $\vO$ the vectorization of the matrix $\mO=\mathbf{J}_N - \I_N$. The product $\IPSDM_{(\K_m)}\SEL$ sets to zero the diagonal terms of the slices $\IPSDM_k, \, k\in \K_m$, and we can then enhance sparsity only over the off-diagonal fibers of $\IPSDT_{\K_m}, \, m \in [K]$.
Then, to learn the $K$-PCG, we formulate the following optimization problem:
\begin{equation}\label{eq:prob1}
\begin{aligned}
    \min_{\{\PSDM_k,\IPSDM_k \in \Se\}} &  \sum_{k\in \positivefrequency} \! \norm{\PSDM_k \IPSDM_k \! - \!\I_N}_{\infty} \!\!\! + \! \lambda \! \! \sum_{m \in [K]}\!\norm{\IPSDM_{(\K_m)}\SEL}_{2,1}\\
    \textrm{subject to} &\;\; \norm{\PSDM_k-\inPSDM_k}_F^2\leq \eta, \; \forall k \in \positivefrequency.
\end{aligned}
\tag{P2}
\end{equation}
The first term in the objective function of \eqref{eq:prob1} enforces $\IPSDM_k$ to be the inverse of $\PSDM_k$ over the frequency domain. Indeed, the first term vanishes if $\IPSDM_k$ is the inverse of $\PSDM_k$. The second term of the objective in \eqref{eq:prob1} promotes block-sparsity on $K$ frequency bands $\{\K_m\}_{m\in[K]}$ (i.e., block-sparsity on the columns of $\IPSDM_{(\K_m)}\SEL$, for all $m\in[K]$), in accordance with our assumption. Here, $\lambda \in \R_+$ is a regularization parameter. The inequality constraints in \eqref{eq:prob1} constrain each slice of the estimate of the CSD matrices to be close to the corresponding slice of the smoothed periodogram (cf. \Cref{sec:preliminaries}). As a result, the matrices $\{\PSDM_k\}$ learned by our method can deviate from the smoothed periodogram during the learning process. The magnitude of this deviation is controlled by $\eta \in \R_+$, which represents a tuning parameter. Finally, to ensure that the learned matrices are Hermitian and positive-definite, we optimize $\PSDM_k$ and $\IPSDM_k$ over $\Se$ defined in (\ref{eq:Se}). 

Unfortunately, \eqref{eq:prob1} is nonconvex due to the presence of the bi-linear terms $\sum_{k\in \positivefrequency} \! ||\PSDM_k \IPSDM_k \! - \!\I_N||_{\infty}$ in the objective function. To handle its non-convexity, in the sequel, we adopt an efficient algorithmic framework based on a version of the \emph{successive convex approximation} (SCA,~\cite{scutari2016parallel}) scheme, which can be thought as a generalization of~\cite{razaviyayn2014parallel}. The proposed algorithm finds a (local) solution of the original nonconvex problem (P2) by solving a sequence of strongly convex subproblems, where the original nonconvex objective is replaced by an appropriate (strongly) convex approximation, which is detailed in the sequel for our case. SCA is suitable for distributed optimization, ensures feasibility at each iteration, and convergence to stationary points is guaranteed under mild assumptions~\cite{scutari2016parallel}.

\subsection{Building the strongly convex surrogate problem}\label{subsec:covexsurrogate}

The first step of the SCA framework entails the definition of a proper strongly convex surrogate problem, which approximates \eqref{eq:prob1} around a given point. To this aim, let us denote the bilinear function in \eqref{eq:prob1} as: 
\begin{equation}\label{eq:fI} 
    f^I(\{\PSDM_k\},\{\IPSDM_k\})=\sum_{k\in \positivefrequency} \! \norm{\PSDM_k \IPSDM_k \! - \!\I_N}_{\infty}.    
\end{equation}
Then, we proceed designing a strongly convex surrogate for the nonconvex term $f^I(\{\PSDM_k\}, \{\IPSDM_k\})$ in (\ref{eq:fI}), satisfying some analytical conditions \cite{scutari2016parallel}. Let us define $\mathbf{Z}_k:=(\PSDM_k , \IPSDM_k )$ for $k \in \positivefrequency$, and let $\mathbf{Z}^t_k:=(\PSDM^t_k , \IPSDM^t_k )$ be the iterate at time $t$. Now, taking into account the separability over $k$ and hinging on the bilinear structure of (\ref{eq:fI}) (see \cite{scutari2016parallel}), we define the (strongly) convex surrogate function around the point $\mathbf{Z}^t_k$:
\begin{equation}\label{eq:surrogate0}
    \begin{aligned}
        &\widetilde{f}^I(\{\mathbf{Z}_k\};\{\mathbf{Z}^t_k\}) \!=\! \sum_{k\in\mathcal{F}} \Big( \norm{\PSDM_k\IPSDM_k^t \!-\! \I_N}_{\infty} + \\
        &\;\;+ \norm{\PSDM_k^t \IPSDM_k \!-\! \I_N}_{\infty}+ \dfrac{\tau}{2}\norm{\PSDM_k \!-\! \PSDM_k^t}^2 \!+\! \dfrac{\tau}{2}\norm{\IPSDM_k \!-\! \IPSDM_k^t}^2 \Big)\, ,
    \end{aligned}
\end{equation}
with $\tau \in \R_+$. It is easy to see that the surrogate in (\ref{eq:surrogate0}) is strongly convex and satisfies
$$\nabla_{\mathbf{Z}^*} f^I \at[\big]{\{\mathbf{Z}^t_k\}} = \nabla_{\mathbf{Z}^*} \widetilde{f}^I \at[\big]{\{\mathbf{Z}^t_k\}},$$
which, using the \emph{Wirtinger calculus}~\cite{fischer2005precoding} and considering that \Cref{eq:fI,eq:surrogate0} are real functions, guarantees first-order optimality conditions, i.e., every stationary point of $\widetilde{f}^I$ in (\ref{eq:surrogate0}) is also a stationary point of the nonconvex objective function $f^I$ in (\ref{eq:fI}). Let us now rewrite (\ref{eq:surrogate0}) and then (\ref{eq:prob1}) in an equivalent form, which is more convenient for our mathematical derivations. Specifically, given two matrices $\mathbf{A} \in \C^{N_1 \times N_2}$ and $\mathbf{B}\in \C^{N_2 \times N_3}$, the vectorization of their product reads as 
\begin{equation}\label{eq:vectorization}
    \mathrm{vec(\mathbf{A}\mathbf{B})}\stackrel{(a)}{=}(\I_{N_3} \otimes \mathbf{A})\underline{\mathbf{b}}\stackrel{(b)}{=}(\mathbf{B}^\top \otimes \I_{N_1})\underline{\mathbf{a}}\,.
\end{equation}
Thus, using \eqref{eq:vectorization} in (\ref{eq:surrogate0}), and exploiting the equivalence $\|\mathbf{A}\|_{\infty}=\|{\rm vec}(\mathbf{A})\|_{\infty}$ of the element-wise infinity norm, we obtain
\begin{equation}\label{eq:surrogate}
    \begin{aligned}
        &\widetilde{f}^I(\{\vZ_k\};\{\vZ^t_k\}) = \sum_{k\in\mathcal{F}} \Big(\norm{(\IPSDM_k^{t^\top} \! \otimes \! \I_{N})\vF_k \!-\! \vI_N}_{\infty}\!+\\
        &+\! \dfrac{\tau}{2}\norm{\vF_k \!-\! \vF_k^t}^2 \!\!+\! \norm{(\I_{N} \! \otimes \! \PSDM_k^t)\vP_k \!-\! \vI_N}_{\infty} \!\!\!+\! \dfrac{\tau}{2}\norm{\vP_k \!-\! \vP_k^t}^2 \Big).
    \end{aligned}
\end{equation}
Then, we conveniently rewrite the block-sparsity inducing term in the objective of ~\eqref{eq:prob1} as:
\begin{align}\label{eq:vec_block_norm}
    &\sum_{m \in [K]}\!\norm{\IPSDM_{(\K_m)}\SEL}_{2,1}
    \stackrel{(a)}{=} \sum_{m \in [K]}\!\sum_{j \in [N^2]} \!\norm{\left[ {\rm vec}\left(\IPSDM_{(\K_m)}\SEL\right)    \right]_j}_{2} \nonumber\\
    &\qquad\quad \stackrel{(b)}{=}\sum_{m \in [K]}\!\sum_{j \in [N^2]} \!\norm{\left[(\SEL \otimes \I_{|\K_m|})\vP_{(\K_m)}\right]_j}_{2}
\end{align}
where $(a)$ follows from the definition of the $\ell_{2,1}$-norm (cf. \Cref{sec:intro}), where the $j$-th block of $\left[ {\rm vec}\left(\IPSDM_{(\K_m)}\SEL\right)\right]_j$ corresponds to the $j$-th column of $\IPSDM_{(\K_m)}\SEL$, for all $j\in[N^2]$; and $(b)$ simply follows from the application of (\ref{eq:vectorization}b), having introduced $\vP_{(\K_m)}={\rm vec}(\IPSDM_{(\K_m)})\in \mathbb{C}^{|\mathcal{K}_m|N^2}$. Finally, using (\ref{eq:surrogate}) and (\ref{eq:vec_block_norm}), we apply the SCA framework replacing the original nonconvex problem ~\eqref{eq:prob1} with a sequence of strongly convex subproblems that, at each iteration $t$, read as:
\begin{equation}\label{eq:prob2}
\begin{aligned}
    \min_{\{\vF_k, \vP_k\in \vSe\}} &\widetilde{f}^I(\{\vZ_k\};\{\vZ^t_k\})+ \\
    &+\! \lambda \! \! \sum_{m \in [K]}\!\sum_{j \in [N^2]} \!\norm{\left[(\SEL \otimes \I_{|\K_m|})\vP_{(\K_m)}\right]_j}_{2}\\
    \textrm{subject to} &\;\; \norm{\PSDM_k-\inPSDM_k}_F^2\leq \eta, \;\;\; \forall k \in \positivefrequency.
\end{aligned}
\tag{$\widetilde{\text{P2}}$}
\end{equation} 
The SCA iterative procedure then follows from computing the (unique) solution of \eqref{eq:prob2}, i.e., $\{\vF^{t+1}_k\}$ and $\{\vP^{t+1}_k\}$, and then smoothing it by using a diminishing stepsize rule~\cite{scutari2016parallel}. 
Interestingly, the surrogate problem \eqref{eq:prob2} is separable over $\vF_k$ and $\vP_k$. Therefore, \eqref{eq:prob2} can be split into two convex subproblems, one involving $\{\vP_k\}$, the other $\{\vF_k\}$. These two subproblems could be solved using off-the-shelf solvers for convex optimization, e.g., \cite{grant2014cvx,diamond2016cvxpy}. However, this approach can be computationally demanding, especially when the size of the involved variables becomes moderately large. Thus, to reduce the overall computational burden of our algorithm, we solve the two subproblems {\it inexactly} by combining the SCA framework \cite{scutari2018parallel} and the \emph{alternating direction method of multipliers} (ADMM, \cite{boyd2011distributed}). Specifically, instead of finding the exact solutions of Problem \eqref{eq:prob2}, we perform an {\it inexact} update by performing one ADMM iteration on the two convex subproblems involving $\{\vP_k\}$ and $\{\vF_k\}$, respectively. The clear advantage is that the recursions of the inner ADMM are given in closed-form (cf. \Cref{alg:nonconvex}, \Cref{subsec:subproblemP,subsec:subproblemF}), and are amenable for parallel implementation, thus largely reducing the overall computation burden. Furthermore, our approach falls within the framework of inexact SCA, whose convergence properties have been studied in \cite{scutari2016parallel}. In the sequel, we will provide a detailed derivation of the (single) ADMM recursion needed to solve the two subproblems. 

\subsection{Subproblem in \texorpdfstring{$\{\vP_k\}$}{TEXT}}\label{subsec:subproblemP}
We start from the subproblem in $\{\vP_k\}$. 
From \eqref{eq:prob2},  (\ref{eq:surrogate}), and exploiting the indicator function $I_{\vSe}(\cdot)$ of \eqref{eq:Se} (useful to induce the positive definiteness of the slices $\{\IPSDM_k\}$), we have to solve
\begin{equation}\label{eq:prob2.1}
\begin{aligned}
    \min_{\substack{\{\vP_k\}}} \; & \sum_{k\in\mathcal{F}} \Big(\norm{(\I_{N} \! \otimes \! \PSDM_k^t)\vP_k \!-\! \vI_N}_{\infty} +\dfrac{\tau}{2}\norm{\vP_k \!-\! \vP_k^t}^2 \!+\\
    &\!+ I_{\vSe}(\vP_k) \!\Big) \!+\! \lambda \! \! \sum_{m \in [K]}\!\sum_{j \in [N^2]} \!\norm{\!\left[\!(\SEL \!\otimes\! \I_{|\K_m|})\vP_{(\K_m)}\!\right]_j\!}_{2}.
\end{aligned}
\tag{$\widetilde{\text{P2}}.1$}
\end{equation}
To deal with the non-smooth terms associated with the $\ell_{\infty}$-norm, the indicator function, and the sum of $\ell_{2}$-norms in \eqref{eq:prob2.1}, we introduce three (sets of) splitting variables. The first, $\{\vX_k \in \C^{N^2}\}$, handles the $\ell_{\infty}$-norm and is defined as 
\begin{equation}\label{eq:vX}
    \vX_k=(\I_{N} \otimes \PSDM_k^t)\vP_k - \vI_N, \; \forall k \in \positivefrequency.
\end{equation}
The second, $\{\vW_k \in \C^{N^2}\}$, handles the indicator function and reads as
\vspace{-.2cm}
\begin{equation}\label{eq:vW}
    \vW_k=\vP_k, \; \forall k \in \positivefrequency.
\end{equation}
The third, $\vV=[\vV_{\K_1}^\top,\ldots,\vV_{\K_K}^\top]^\top$, $\vV \in \C^{N^2 M}$, handles the sum of $\ell_{2}$-norms in \eqref{eq:prob2.1}, and is given by
\vspace{-.05cm}
\begin{equation}\label{eq:vV}
    \vV_{\K_m}=\left(\SEL \otimes \I_{|\K_m|}\right)\vP_{(\K_m)}.
\end{equation}
Now,  using \eqref{eq:vX}-\eqref{eq:vV} in \eqref{eq:prob2.1}, we obtain the equivalent problem 
\begin{equation}\label{eq:equivP}
\begin{aligned}
    \min_{ \{ \vP_k, \vX_k, \vW_k\}, \vV} &\; \sum_{k\in\mathcal{F}} \Big( \norm{\vX_k}_{\infty}\!+\! \dfrac{\tau}{2}\norm{\vP_k \!-\! \vP_k^t}^2 \!+\! I_{\vSe}(\vW_k)\Big) \!+\\
    &+\!\lambda\sum_{m \in [K]}\!\sum_{j \in [N^2]} \!\norm{\left[\vV_{\K_m}\right]_j}_{2} \\
    \textrm{subject to} \quad & \;(\I_{N} \!\otimes\! \PSDM_k^t)\vP_k \!-\! \vI_N \!-\! \vX_k\!=\!\mathbf{0}, \, \forall k \!\in\! \positivefrequency,\\
    & \; \vP_k - \vW_k = \mathbf{0}, \, \forall k \!\in\! \positivefrequency,\\
    & \;\left(\SEL \!\otimes\! \I_{|\K_m|}\right)\vP_{(\K_m)} \!-\! \vV_{\K_m}\!=\!\mathbf{0}, \, \forall m \in [K].
\end{aligned}
\tag{$\widehat{\text{P2}}.1$}
\end{equation}
The scaled augmented Lagrangian of (\ref{eq:equivP}) reads as \cite{boyd2011distributed}:
\begin{equation}\label{eq:AULP}
    \begin{aligned}
        &\mathcal{L}_{\sigma,\omega,\theta}(\{\vP_k\}, \{\vX_k\}, \{\vW_k\}, \vV, \{\bar{\boldsymbol{\mu}}_k\}, \{\bar{\boldsymbol{\omega}}_k\}, \bar{\boldsymbol{\phi}}; \{\vZ_k^t\})=\\
        &+\!\sum_{k\in\mathcal{F}} \!\Bigg(\norm{\vX_k}_{\infty}\!+\!\dfrac{\tau}{2}\norm{\vP_k \!-\! \vP_k^t}^2\!+\! I_{\vSe}(\vW_k) +\\
        &+\! \dfrac{\omega}{2}\norm{\vP_k\!\!-\!\vW_k\!\!+\!\bar{\boldsymbol{\omega}}_k}^2 \!\!\!+\!\dfrac{\sigma}{2} \norm{(\I_{N} \!\otimes\! \PSDM_k^t)\vP_k \!\! - \! \vI_N \!-\! \vX_k \!\!+\! \bar{\boldsymbol{\mu}}_k}^2\Bigg)+ \\
        &+ \! \sum_{m \in [K]}\!\sum_{j \in [N^2]} \!\Bigg(\lambda\norm{\left[\vV_{\K_m}\right]_j}_{2} \!\!+\\
        &+\dfrac{\theta}{2} \norm{\left[\left(\SEL \!\otimes\! \I_{|\K_m|}\right)\vP_{(\K_m)}\right]_j \!-\! \left[\vV_{\K_m}\right]_j \!+\! \left[\bar{\boldsymbol{\phi}}_{\K_m}\right]_j}^2\Bigg).
    \end{aligned}
\end{equation}
The additional strongly convex terms in \eqref{eq:AULP} have as penalty parameters $\sigma, \, \omega, \, \theta$, all belonging to $\R_+$, and as scaled dual variables the sequence of vectors $\{\bar{\boldsymbol{\mu}}_k \in \C^{N^2}\}$, $\{\bar{\boldsymbol{\omega}}_k \in \C^{N^2}\}$, and $\bar{\boldsymbol{\phi}}=[\bar{\boldsymbol{\phi}}_{\K_1}^\top, \ldots, \bar{\boldsymbol{\phi}}_{\K_K}^\top]^\top \in \C^{N^2 M}$, with $\bar{\boldsymbol{\phi}}_{\K_m} \in \C^{N^2|\K_m|}$. Then, ADMM proceeds by searching a saddle point of the scaled augmented Lagrangian in (\ref{eq:AULP}), through the following recursions \cite{boyd2011distributed}:
\begin{equation}\label{eq:upP}
    \begin{aligned}
        &\text{For all } k \in \positivefrequency\text{:}\\
        \vP_k^{t+1} &= \argmin_{\vP_k} \mathcal{L}_{\sigma,\omega,\theta}(\vP_k; \vX_k^t, \vW_k^t, \vV^t, \bar{\boldsymbol{\mu}}_k^t, \bar{\boldsymbol{\omega}}_k^t, \bar{\boldsymbol{\phi}}^t, \vZ_k^t),\\
        \vX_k^{t+1} &= \argmin_{\vX_k} \mathcal{L}_{\sigma,\omega,\theta}(\vX_k; \vP_k^{t+1}, \vW_k^t, \vV^t, \bar{\boldsymbol{\mu}}_k^t, \bar{\boldsymbol{\omega}}_k^t, \bar{\boldsymbol{\phi}}^t, \vZ_k^t),\\
        \vW_k^{t+1} &= \argmin_{\vW_k} \mathcal{L}_{\sigma,\omega,\theta}(\vW_k; \vP_k^{t+1}\!\!, \vX_k^{t+1}\!\!, \vV^t\!, \bar{\boldsymbol{\mu}}_k^t, \bar{\boldsymbol{\omega}}_k^t, \bar{\boldsymbol{\phi}}^t\!, \vZ_k^t),\\
        \bar{\boldsymbol{\mu}}_k^{t+1} &= \bar{\boldsymbol{\mu}}_k^{t} \!+\! \left( (\I_{N} \otimes \PSDM_k^t)\vP_k^{t+1} - \vI_N - \vX_k^{t+1} \right),\\
        \bar{\boldsymbol{\omega}}_k^{t+1} &= \bar{\boldsymbol{\omega}}_k^{t} \!+\! \left(\vP_k^{t+1}-\vW_k^{t+1}\right),\\
        &\text{For all } m \in [K]\text{:}\\
        \vV_{\K_m}^{t+1} &= \argmin_{\vV_{\K_m}} \mathcal{L}_{\sigma,\omega,\theta}(\vV_{\K_m}; \vP_k^{t+1}, \vX_k^{t+1}, \bar{\boldsymbol{\mu}}_k^{t+1}, \bar{\boldsymbol{\phi}}^t, \vZ_k^t),\\
        \bar{\boldsymbol{\phi}}_{\K_m}^{t+1} &= \bar{\boldsymbol{\phi}}_{\K_m}^{t} + \left(\left(\I_{|\K_m|}\otimes \SEL \right)\vP_{(\K_m)}^{t+1} - \vV_{\K_m}^{t+1} \right).
    \end{aligned}
    \tag{S1}
\end{equation}
Interestingly, the first, second, third, and sixth steps in (\ref{eq:upP}) can be derived in closed-form, as we show in the next paragraphs.

\subsubsection{Solution for \texorpdfstring{$\vP_k^{t+1}$}{TEXT}}\label{subsub:updateP}
The solution of the first step of \eqref{eq:upP} is made explicit in \Cref{lem:updateP}.
\begin{restatable}{lemma}{updateP}\label{lem:updateP}
    The update $\vP_k^{t+1}$ in \eqref{eq:upP} is equal to
    \begin{equation}\label{eq:updateP}
        \begin{aligned}
            &\vP_k^{t+1} = \boldsymbol{\Gamma}_1^{-1} \bigg(\tau\vP_k^{t} + \omega(\!\vW_k^t\!-\!\bar{\boldsymbol{\omega}}_k^t) +\\
            &+\!\sigma\left(\I_{N} \!\otimes\! \PSDM_k^t \right)^\conjugateT\! \left(\vI_N \!+\!\vX_k^{t} \!-\! \bar{\boldsymbol{\mu}}_k^{t}\right)\!+\!\dfrac{\theta}{2}\SEL^\top\!\left( \vV_k^{t} \!-\!\bar{\boldsymbol{\phi}}_k^{t} \right)\!\bigg)\, ,
        \end{aligned}
    \end{equation}
    with $\boldsymbol{\Gamma}_1$ defined as in \eqref{eq:Gamma1}.
\end{restatable}

\begin{proof}
    From \eqref{eq:AULP} and \eqref{eq:upP}, we start rearranging the term 
    \begin{equation}\label{eq:termKm}
        \dfrac{\theta}{2}\!\sum_{m \in [K]}\sum_{j \in [N^2]}  \norm{\left[\left(\SEL\!\otimes\! \I_{|\K_m|} \right)\vP_{(\K_m)}\right]_j \!-\! \left[\vV_{\K_m}^t\right]_j \!+\! \left[\bar{\boldsymbol{\phi}}_{\K_m}^t\right]_j}^2
    \end{equation}
    to make the dependence on $\vP_k$ explicit. Specifically, let us denote the matrix form of $\vV_{\K_m}^t$ with $\widehat{\mathbf{V}}_{(\K_m)}^t \in \C^{|\K_m| \times N^2}$, and of $\bar{\boldsymbol{\phi}}_{\K_m}^t$ with $\bar{\boldsymbol{\Phi}}_{(\K_m)}^t \in \C^{|\K_m| \times N^2}$. At this point, the matrix form of $\vV^t$ reads as $\widehat{\mathbf{V}}^t=[\widehat{\mathbf{V}}_{(\K_1)}^{t^\top}, \ldots, \widehat{\mathbf{V}}_{(\K_K)}^{t^\top}]^\top$, $\widehat{\mathbf{V}}^t\in \C^{M \times N^2}$.
    Analogously, $\bar{\boldsymbol{\Phi}}^t=[\bar{\boldsymbol{\Phi}}_{(\K_1)}^{t^\top}, \ldots, \bar{\boldsymbol{\Phi}}_{(\K_K)}^{t^\top}]^\top$, $\bar{\boldsymbol{\Phi}}^t\in \C^{M \times N^2}$.
    Now, consider $\mathrm{vec}(\widehat{\mathbf{V}}^{t^\top})=[\vV_0^{t^\top},\ldots,\vV_k^{t^\top},\ldots,\vV_{M-1}^{t^\top}]^\top$, $\vV_k^t \in \C^{N^2}$; and $\mathrm{vec}(\bar{\boldsymbol{\Phi}}^{t^\top})=[\bar{\boldsymbol{\phi}}_0^{t^\top},\ldots,\bar{\boldsymbol{\phi}}_k^{t^\top},\ldots,\bar{\boldsymbol{\phi}}_{M-1}^{t^\top}]^\top$, $\bar{\boldsymbol{\phi}}_k^t \in \C^{N^2}$. 
    Hence, we conveniently rewrite \eqref{eq:termKm} as 
    \begin{equation}\label{eq:termk}
        \dfrac{\theta}{2}\sum_{k \in \positivefrequency} \norm{\SEL\vP_k \!-\! \vV_k^t \!+\! \bar{\boldsymbol{\phi}}_k^t}^2.
    \end{equation}
Thus, using (\ref{eq:termk}) in \eqref{eq:AULP}, the first step of (S1) entails the minimization of \eqref{eq:AULP} with respect to $\vP_k$, which writes as:
\begin{align}\label{eq:prob_Pk}
  &\vP_k^{t+1} = \argmin_{\vP_k} \;\dfrac{\tau}{2}\norm{\vP_k \!-\! \vP_k^t}^2 +\dfrac{\omega}{2}\norm{\vP_k\!\!-\!\vW_k^t\!\!+\!\bar{\boldsymbol{\omega}}_k^t}^2+  \nonumber\\
  &\qquad\qquad+\dfrac{\sigma}{2} \norm{(\I_{N} \!\otimes\! \PSDM_k^t)\vP_k \!\! - \! \vI_N \!-\! \vX_k^t \!\!+\! \bar{\boldsymbol{\mu}}_k^t}^2+   \nonumber\\
  &\qquad\qquad +\dfrac{\theta}{2}\sum_{k\in\mathcal{F}}\! \norm{\SEL\vP_k \!-\! \vV_k^t \!+\! \bar{\boldsymbol{\phi}}_k^t}^2.
\end{align}
Now, from \eqref{eq:prob_Pk}, according to \emph{Wirtinger calculus}~\cite{fischer2005precoding}, the stationarity condition follows from imposing the derivative of the objective of \eqref{eq:prob_Pk} w.r.t. $\vP_k^*$ equal to zero.  Mathematically, this is equivalent to:
\begin{equation*}
    \begin{aligned}
            0 &= \tau\, (\vP_k-\vP_k^{t}) + \omega(\vP_k-\vW_k^t+\bar{\boldsymbol{\omega}}_k^t)+ \\
            &+ \sigma\,(\I_{N} \otimes \PSDM_k^t)^\conjugateT \bigg( \left(\I_{N} \otimes \PSDM_k^t \right)\vP_k - \vI_N -\vX_k^{t} + \bar{\boldsymbol{\mu}}_k^{t} \bigg) + \\
            &+\theta\,\SEL^\top\left(\SEL\vP_k - \vV_k^{t} + \bar{\boldsymbol{\phi}}_k^{t} \right)\,.
    \end{aligned}
\end{equation*}
Finally, we obtain
\begin{equation*}
    \begin{aligned}
            \vP_k^{t+1} &= \boldsymbol{\Gamma}_1^{-1} \bigg(\tau\vP_k^{t} + \omega(\!\vW_k^t\!-\!\bar{\boldsymbol{\omega}}_k^t) +\\
            &+\!\sigma\left(\I_{N} \!\otimes\! \PSDM_k^t \right)^\conjugateT\! \left(\vI_N \!+\!\vX_k^{t} \!-\! \bar{\boldsymbol{\mu}}_k^{t}\right)\!+\!\theta\SEL^\top\!\left( \vV_k^{t} \!-\!\bar{\boldsymbol{\phi}}_k^{t} \right)\!\bigg)\,,
    \end{aligned}
\end{equation*}
with $\boldsymbol{\Gamma}_1$ given by
\begin{equation}\label{eq:Gamma1}
    \begin{aligned}
            \boldsymbol{\Gamma}_1 &= \!\left(\tau+\omega\right)\I_{N^2} \!+\! \sigma\! \left(\I_{N} \!\otimes\! \PSDM_k^t\right)^\conjugateT\!\!\left(\I_{N} \!\otimes\! \PSDM_k^t\right) \!+\! \theta \SEL\,,
    \end{aligned}
\end{equation}
where we exploited $\SEL^\top\SEL=\SEL$. From \eqref{eq:Gamma1}, it is easy to see that $\boldsymbol{\Gamma}_1 \in \mathcal{B}_{++}^{N^2}$, with $\mathcal{B}_{++}^{N^2}$ denoting the set of positive definite block-diagonal matrices on $\C^{N^2 \times N^2}$. Thus, the inverse of $\boldsymbol{\Gamma}_1$ exists and can be efficiently computed in a block-wise fashion. This concludes the proof of \Cref{lem:updateP}.
\end{proof}

\subsubsection{Solution for \texorpdfstring{$\vX_k^{t+1}$}{TEXT}}\label{subsub:updateX}

From (\ref{eq:AULP}), letting
\begin{equation}\label{eq:x}
    \mathbf{x}_k = \left(\I_{N}\otimes\PSDM_k^t \right)\vP_k^{t+1} - \vI_N + \bar{\boldsymbol{\mu}}_k^t \, ,    
\end{equation}
the update for $\vX_k^{t+1}$ in \eqref{eq:upP} reads as
\begin{align}\label{eq:Sol_X}
    \vX_k^{t+1} &= \;\argmin_{\vX_k} \; \dfrac{\sigma}{2} \norm{\mathbf{x}_k - \vX_k }^2 + \!\norm{\vX_k}_{\infty} \nonumber\\
    &= \;\mathsf{prox}_{\tfrac{1}{\sigma}\norm{\cdot}_\infty} \left( \mathbf{x}_k\right)\, ; 
\end{align}
where $\mathsf{prox}_{\norm{\cdot}_\infty}$ in (\ref{eq:Sol_X}) denotes the proximal operator of the infinity norm \cite{boyd2011distributed}. In particular, the result of the proximal operator in (\ref{eq:Sol_X}) can be found explicitly by hinging on the following lemma.
\begin{restatable}{lemma}{complexprox}\label{lem:complexprox}
     For any norm function $f=\norm{\cdot}$ on $\C^N$, it holds
     $$\mathsf{prox}_{\widehat{\lambda} f}(\mathbf{z})=\mathrm{sign}(\mathbf{z}) \left(\abs{\mathbf{z}}\!-\! \widehat{\lambda} \Pi_{\mathcal{B}}\left(\dfrac{\abs{\mathbf{z}}}{\widehat{\lambda}} \right)\right)\, ,$$ 
with $\mathbf{z}\in \C^N$, $\widehat{\lambda} \in \R_+$, and $\Pi_{\mathcal{B}}$ denoting the projection operator onto the unit ball of the dual norm $f^*$.
\end{restatable}
\begin{proof}
This directly follows from Moreau decomposition~\cite{parikh2014proximal}:
\begin{equation*}
    \begin{aligned}
        \mathsf{prox}_{\widehat{\lambda} f}(\mathbf{z}) &= \mathbf{z} \!-\! \widehat{\lambda} \Pi_{\mathcal{B}}\left(\dfrac{\mathbf{z}}{\widehat{\lambda}}\right) = \mathrm{sign}(\mathbf{z}) \left(\abs{\mathbf{z}}\!-\! \widehat{\lambda} \Pi_{\mathcal{B}}\left(\dfrac{\abs{\mathbf{z}}}{\widehat{\lambda}} \right)\right)\, ; 
    \end{aligned}
\end{equation*}
with $\mathbf{z}\in \C^N$, $\widehat{\lambda} \in \R_+$, and $\Pi_{\mathcal{B}}$ being the projection operator onto the unit ball of the dual norm $f^*$.
\end{proof}
Using \Cref{lem:complexprox} in (\ref{eq:Sol_X}), we get
\begin{equation}\label{eq:upXprox}
    \vX_k^{t+1} =\mathrm{sign}(\mathbf{x}_k)\left(\abs{\mathbf{x}_k} - \dfrac{1}{\sigma} \Pi_{\mathcal{B}(\sigma)}\left( \sigma\abs{\mathbf{x}_k}\right)\right),
\end{equation}
where $\Pi_{\mathcal{B}(\sigma)}$ is the projection operator onto the $\ell_1$-norm ball of radius $\sigma$, being the $\ell_1$- the dual of the $\ell_\infty$-norm \cite{parikh2014proximal}. 

\subsubsection{Solution for \texorpdfstring{$\vW_k^{t+1}$}{TEXT}}\label{subsub:updateW}
Here, we first derive the solution in matrix form, and then we vectorize it. From \eqref{eq:AULP}, letting 
  $  \mathbf{W}_k = \IPSDM_k^{t+1}+\bar{\boldsymbol{\Omega}}_k^{t}$,
with $\bar{\boldsymbol{\Omega}}_k^{t}$ being the $N \times N$ matrix representation of $\bar{\boldsymbol{\omega}}_k^{t}$, the update for $\W_k^{t+1}$ in \eqref{eq:upP} is given by
\begin{equation}\label{eq:probW}
    \begin{aligned}
        \W_k^{t+1} &= \argmin_{\W_k} \;\;\dfrac{\omega}{2}\norm{\mathbf{W}_k\!-\!\W_k}^2 \!\!+\! I_{\Se}(\W_k)\\
        &= \argmin_{\W_k \in \Se} \;\;\norm{\mathbf{W}_k-\W_k}^2.
    \end{aligned}
\end{equation}
From \cite{parikh2014proximal}, the solution of \eqref{eq:probW} corresponds to projecting the Hermitian matrix $\bar{\mathbf{W}}_k = (\mathbf{W}_k +\mathbf{W}_k^\conjugateT)/2$ onto $\Se$, where the projection reads as 
\begin{equation}
\begin{aligned}
    & \mathring{\mathbf{W}}_k = \Pi_{\Se}(\bar{\mathbf{W}}_k) = \sum_{i\in [N]} \max(\bar{w}_i,\epsilon)\,\mathring{\mathbf{w}}_i\mathring{\mathbf{w}}^\conjugateT_i,\\
    &\text{with}\quad \bar{\mathbf{W}}_k=\sum_i \bar{w}_i\,\mathring{\mathbf{w}}_i\mathring{\mathbf{w}}^\conjugateT_i \, .
\end{aligned}
\end{equation}
Finally, we set 
\begin{equation}\label{eq:upW}
    \vW_k^{t+1}=\mathrm{vec}(\mathring{\mathbf{W}}_k) \, .    
\end{equation}

\subsubsection{Solution for \texorpdfstring{$\vV_{\K_m}^{t+1}$}{TEXT}}\label{subsub:updateV}
From (\ref{eq:AULP}), the update for $\vV_{\K_m}$ in \eqref{eq:upP} reads as 
\begin{align}\label{eq:upVfull}
        &\vV_{\K_m}^{t+1}  = \argmin_{\vV_{\K_m}}\;\; \hspace{-.1cm} \sum_{j \in [N^2]} \!\Bigg(\lambda\norm{\left[\vV_{\K_m}\right]_j}_{2} \!\!+ \nonumber\\
        &\!+\!\dfrac{\theta}{2} \norm{\left[\left(\SEL \!\otimes\! \I_{|\K_m|}\right)\vP_{(\K_m)}^{t+1}\right]_j \!-\! \left[\vV_{\K_m}\right]_j \!+\! \left[\bar{\boldsymbol{\phi}}_{\K_m}^t\right]_j}^2\Bigg)\, .
\end{align}
Let us define
\begin{equation}\label{eq:v}
    \left[\mathbf{v}_{\K_m}\right]_j = \left[\left(\SEL \otimes \I_{|\K_m|}  \right)\vP^{t+1}_{(\K_m)}\right]_{j} \!+\! \left[\bar{\boldsymbol{\phi}}^t_{\K_m}\right]_{j}.    
\end{equation}
Now, exploiting the separability of the objective in \eqref{eq:upVfull} over the index $j$, we obtain:
\begin{align}\label{eq:upV}
        \left[\vV_{\K_m}^{t+1}\right]_j  &= \argmin_{\left[\vV_{\K_m}\right]_j } \;\dfrac{\theta}{2} \norm{\left[\mathbf{v}_{\K_m}\right]_j \!-\! \left[\vV_{\K_m}\right]_j}^2\!\!\!+\lambda\,\norm{\left[\vV_{\K_m}\right]_j}_{2}\nonumber\\
        &=\mathsf{prox}_{\tfrac{\lambda}{\theta}\norm{\cdot}_2} \left( \left[\mathbf{v}_{\K_m}\right]_j \right)\, ;
\end{align}
for all $j \in [N^2]$.
Finally, exploiting Lemma \ref{lem:complexprox} in (\ref{eq:upV}), we get
\begin{equation}\label{eq:upVprox}
    \left[\vV_{\K_m}^{t+1}\right]_j = \mathrm{sign}(\left[\mathbf{v}_{\K_m}\right]_j)\mathrm{B}_{\lambda/\theta}\left( \abs{\left[\mathbf{v}_{\K_m}\right]_j} \right),
\end{equation}
where $\mathrm{B}_{\widetilde{\lambda}}(\mathbf{z})=\max\left(0,1- \widetilde{\lambda}/\norm{\mathbf{z}}_2 \right) \mathbf{z}$ is the block soft-thresholding operator\cite{parikh2014proximal}. Finally, we get $$\vV_{\K_m}^{t+1}=\left[[\vV_{\K_m}^{t+1}]_1^\top, \ldots, [\vV_{\K_m}^{t+1}]_{N^2}^\top\right]^\top.$$

\subsection{Subproblem in \texorpdfstring{$\{\vF_k\}$}{TEXT}}\label{subsec:subproblemF}
In this case, from \eqref{eq:prob2} and (\ref{eq:surrogate}), we have to solve
\begin{align}\label{eq:prob2.2}
    &\min_{\{\vF_k\}} \,  \sum_{k \in \positivefrequency} \Big(\norm{(\IPSDM_k^{t^\top} \! \otimes \! \I_{N})\vF_k \!-\! \vI_N}_{\infty} \!\!+\! \dfrac{\tau}{2}\norm{\vF_k \!-\! \vF_k^t}^2 \!\!\!+I_{\vSe}(\vF_k)\Big)\nonumber\\
    &\qquad\textrm{subject to} \;  \norm{\vF_k-\vFtilde_k}_2^2\leq \eta, \;\; \forall k \in \positivefrequency.
\tag{$\widetilde{\text{P2}}.2$}
\end{align}
Analogously to \eqref{eq:prob2.1}, we exploit $I_{\vSe}(\vF_k)$ to induce positive definiteness of the slices $\{\vF_k\}$. Now, to manage the non-smooth terms in \eqref{eq:prob2.2}, we introduce two (sets of) splitting variables. The first, $\{\vU_k\in \C^{N^2}\}$, is given by
\begin{equation}\label{eq:vU}
    \vU_k=(\IPSDM_k^{t^\top} \otimes \I_{N})\vF_k - \vI_N, \;\;\forall k \in \positivefrequency.
\end{equation}
The second, $\{\vL_k\in \C^{N^2}\}$, is defined as
\begin{equation}\label{eq:vL}
    \vL_k=\vF_k, \;\; \forall k \in \positivefrequency.
\end{equation}
Thus, using (\ref{eq:vU}) and \eqref{eq:vL} in \eqref{eq:prob2.2}, we equivalently get
\begin{equation}\label{eq:equivF}
\begin{aligned}
    \min_{\substack{\{\vF_k, \vU_k, \vL_k\}}} \; & \sum_{k \in \positivefrequency} \Big(\norm{\vU_k}_{\infty} \!+\! \dfrac{\tau}{2}\norm{\vF_k \!-\! \vF_k^t}^2+I_{\vSe}(\vF_k)\Big)\\
    \;\;\textrm{subject to} \; & \norm{\vF_k-\vFtilde_k}_2^2\leq \eta, \;\; \forall k \in \positivefrequency,\\
    &  (\IPSDM_k^{t^\top} \otimes \I_{N})\vF_k - \vI_N - \vU_k=\mathbf{0}, \;\; \forall k \in \positivefrequency,\\
    & \vF_k - \vL_k=\mathbf{0}, \;\; \forall k \in \positivefrequency.
\end{aligned}
\tag{$\widehat{\text{P2}}.2$}
\end{equation}
Considering for the moment only the equality constraints, the (scaled) augmented Lagrangian of (\ref{eq:equivF}) reads as
\begin{align}\label{eq:aulF}
        &\mathcal{L}_{\rho,\delta}(\{\vF_k\}, \{\vU_k\},\{\vL_k\}, \{\bar{\boldsymbol{\alpha}}_k\}, \{\bar{\boldsymbol{\delta}}_k\}; \{\vZ_k^t\})= \nonumber\\
        &\!\!=\sum_{k \in \positivefrequency} \Bigg( \norm{\vU_k}_{\infty} \!\!\!+\! \dfrac{\tau}{2}\norm{\vF_k \!-\! \vF_k^t}^2 \!+\!I_{\vSe}(\vL_k)+ \nonumber\\
        &\;\;+ \dfrac{\rho}{2} \norm{(\IPSDM_k^{t^\top} \!\!\otimes\! \I_{N})\vF_k \!-\! \vI_N \!-\! \vU_k \!+\! \bar{\boldsymbol{\alpha}}_k}^2\!\!+\!\dfrac{\delta}{2}\norm{\vF_k \!-\! \vL_k \!+\! \bar{\boldsymbol{\delta}}_k}^2 \Bigg)\, ,
\end{align}
where $\rho \in \R_+$ and $\delta \in \R_+$ are penalty parameters;  $\{\bar{\boldsymbol{\alpha}}_k\in \C^{N^2}\}$ and $\{\bar{\boldsymbol{\delta}}_k\in \C^{N^2}\}$ are the scaled dual variables associated with the first and second equality constraints in \eqref{eq:equivF}.
Using the ADMM principle, we proceed looking for a saddle point of the augmented Lagrangian in (\ref{eq:aulF}) that satisfies also the inequality constraints in (\ref{eq:equivF}) \cite{boyd2011distributed,boyd2004convex}. Since \eqref{eq:aulF} is fully separable over $k$, we get the recursions below for each $k \in \positivefrequency$:
\begin{equation}
    \begin{aligned}\label{eq:upF}
        \vF_k^{t+1} \!\!=&\, \arg \min_{\vF_k} \mathcal{L}_{\rho,\delta}(\vF_k; \vU_k^t,\! \vL_k^t,\! \bar{\boldsymbol{\alpha}}_k^t,\! \bar{\boldsymbol{\delta}}_k^t,\! \vZ_k^t)\\
        & \;\;\;\textrm{subject to} \quad \norm{\vF_{k} \!-\! \vFtilde_k}^{2} \leq \eta \, ;\\
        \vU_k^{t+1} \!\!=& \,\argmin_{\vU_k}\mathcal{L}_{\rho,\delta}(\vU_k; \vF_k^{t+1}\!\!, \vL_k^t,\! \bar{\boldsymbol{\alpha}}_k^t,\! \bar{\boldsymbol{\delta}}_k^t,\! \vZ_k^t), \\
        \vL_k^{t+1} \!\!=& \,\argmin_{\vL_k}\mathcal{L}_{\rho,\delta}(\vL_k; \vF_k^{t+1}\!\!, \vU_k^{t+1}\!\!,\! \bar{\boldsymbol{\alpha}}_k^t,\! \bar{\boldsymbol{\delta}}_k^t,\! \vZ_k^t),\\
        \bar{\boldsymbol{\alpha}}_k^{t+1} \!\!=& \,\bar{\boldsymbol{\alpha}}_k^{t} \!+\! \left( (\IPSDM_k^{t^\top} \otimes \I_{N})\vF_k^{t+1} \!-\! \vI_N \!-\! \vU_k^{t+1} \right),\\ 
        \bar{\boldsymbol{\delta}}_k^{t+1} \!\!=& \,\bar{\boldsymbol{\delta}}_k^{t} \!+\! \left(\vF_k^{t+1}-\vL_k^{t+1}\right).
    \end{aligned}
    \tag{S2}
\end{equation}
Interestingly, the first, second, and third updates in (\ref{eq:upF}) can be evaluated in (quasi-)closed-form, as we illustrate in the sequel.

\subsubsection{Solution for \texorpdfstring{$\vF_k^{t+1}$}{TEXT}}\label{subsub:updateF}
Starting from \eqref{eq:aulF}, the first subproblem in (\ref{eq:upF}) reads as:
\begin{align}\label{eq:minmaxF}
        \vF_k^{t+1} =& \arg \min_{\vF_k} \;\; \dfrac{\tau}{2}\norm{\vF_k \!-\! \vF_k^t}^2\!+ \nonumber\\
        &\hspace{0cm}+\!\dfrac{\rho}{2} \norm{(\IPSDM_k^{t^\top} \!\!\otimes\! \I_{N})\vF_k \!-\! \vI_N \!-\! \vU_k^t \!+\! \bar{\boldsymbol{\alpha}}_k^t}^2\!\!+\nonumber\\
        &\hspace{0cm}+\!\dfrac{\delta}{2}\norm{\vF_k \!\!-\! \vL_k^t \!\!+\! \bar{\boldsymbol{\delta}}^t}^2 \nonumber\\
        &\hspace{0cm} \textrm{subject to} \quad \norm{\vF_{k} - \vFtilde_k}^{2}\leq \eta \,.
        \tag{S2.1}
\end{align}
Starting from \eqref{eq:minmaxF}, we build the related Lagrangian
\begin{align}\label{eq:lagrangianF}
    \widetilde{\mathcal{L}}_{\rho,\delta}(\vF_k, \beta_k)=&\dfrac{\tau}{2}\norm{\vF_k \!-\! \vF_k^t}^2\!+ \nonumber\\
    &\hspace{-1cm}+\!\dfrac{\rho}{2} \norm{(\IPSDM_k^{t^\top} \!\!\otimes\! \I_{N})\vF_k \!-\! \vI_N \!-\! \vU_k^t \!+\! \bar{\boldsymbol{\alpha}}_k^t}^2\!\!+\nonumber\\
    &\hspace{-1cm}+\!\dfrac{\delta}{2}\norm{\vF_k \!\!-\! \vL_k^t \!\!+\! \bar{\boldsymbol{\delta}}^t}^2 \!\!\!+\!\beta_k \!\left(\norm{\vF_{k} \!-\! \vFtilde_k}^{2}\!\!-\!\eta \!\right) \,,
\end{align}
where $\beta_k\geq0$ is the dual variable related to the inequality constraint in \eqref{eq:minmaxF}. Thus, exploiting \eqref{eq:lagrangianF}, we compute jointly $\vF_{k}^{t+1}$ and $\beta_k^{t+1}$ by solving the Karush-Kuhn-Tucker (KKT) conditions of \eqref{eq:minmaxF}, which read as \cite{boyd2004convex}:
\begin{equation}\label{eq:kktF}
    \begin{aligned}
        & 0 \in \pdv*{\widetilde{\mathcal{L}}_{\rho,\delta}(\vF_k, \beta_k)}{\vF_k^*},\\
        & 0 \leq \eta - \norm{\vF_{k} - \vFtilde_k}^{2} \perp  \beta_k \geq 0\,.
    \end{aligned}
\end{equation}
The first condition in \eqref{eq:kktF} imposes primal optimality, while the second one is a variational inequality encompassing complementary slackness, primal and dual feasibility. From (\ref{eq:aulF}), imposing the first condition in \eqref{eq:kktF}, we get 
\begin{equation}\label{eq:stationarity_F}
\begin{aligned}
    0 &= \dfrac{\tau}{2}\left(\vF_k - \vF_k^{t}\right) + \beta_k\left(\vF_{k} - \vFtilde_k\right) + \dfrac{\delta}{2}\left(\vF_k - \vL_k^{t}+\bar{\boldsymbol{\delta}}_k^t\right) + \\
    &+ \dfrac{\rho}{2} \left(\IPSDM_k^{t^\top} \otimes \I_{N}\right)^\conjugateT \bigg( \left(\IPSDM_k^{t^\top} \otimes \I_{N}\right)\vF_k - \vI_N - \vU_k^{t} + \bar{\boldsymbol{\alpha}}_k^{t}\bigg)\,.
\end{aligned}
\end{equation}
Thus, exploiting the property of the Kronecker product $(\mathbf{A}\otimes \mathbf{B})(\mathbf{C}\otimes \mathbf{D})=(\mathbf{A}\mathbf{C})\otimes (\mathbf{B}\mathbf{D})$ in (\ref{eq:stationarity_F}), and letting 
\begin{equation}\label{eq:Gamma2}
    \boldsymbol{\Gamma}_2(\beta_k) = \I_N \otimes \bigg( \Big(\dfrac{\tau}{2} + \beta_k + \dfrac{\delta}{2} \Big) \I_{N}+\dfrac{\rho}{2} \IPSDM_k^{t^\conjugateT}\IPSDM_k^{t} \bigg) \in \mathcal{B}_{++}^{N^2}\, ,    
\end{equation}
with $\mathcal{B}_{++}^{N^2}$ denoting the set of positive semidefinite block-diagonal matrices on $\C^{N^2 \times N^2}$, we obtain
\begin{equation}\label{eq:updateF}
    \begin{aligned}
        \vF_{k}(\beta_k) &=\boldsymbol{\Gamma}_2^{-1}(\beta_k)\bigg(\dfrac{\tau}{2}\vF_k^{t} + \beta_k\vFtilde_k + \dfrac{\delta}{2}\left(\vL_k^t-\bar{\boldsymbol{\delta}}_k^{t}\right)\\
        &\;\;+\dfrac{\rho}{2} \left(\IPSDM_k^{t^\top} \otimes \I_{N}\right)^\conjugateT \left(\vI_N + \vU_k^{t} - \bar{\boldsymbol{\alpha}}_k^{t}\right) \bigg)\, .
    \end{aligned}
\end{equation}
Then, we proceed as follows. First, we check whether $\vF_{k}(0)$ satisfies the primal feasibility condition in (\ref{eq:kktF}). If it does, then we set $\vF_k^{t+1}\!\!=\!\vF_k(0)$ and $\beta_k^{t+1}\!\!=\!0$. On the contrary, if the primal feasibility is not satisfied, it means that we need to find $\beta_k^{\star}>0$ as the root of the equation $\norm{\vF_{k}(\beta^{\star}_k) \!-\! \vFtilde_k}^{2}\!=\!\eta$, which is unique since problem \eqref{eq:minmaxF} is strongly convex, and can be efficiently found using the bisection method \cite{burden2015numerical}. Finally, we set 
\begin{equation}\label{eq:upF_and_beta}
    \vF_k^{t+1}=\vF_k(\beta_k^*)\,, \; \mathrm{and}\; \beta_k^{t+1}=\beta_k^\star\,.
\end{equation}

\subsubsection{Solution for \texorpdfstring{$\vU_k^{t+1}$}{TEXT}} 
From \eqref{eq:aulF}, letting
\begin{equation}\label{eq:u}
    \mathbf{u}_k=\left(\IPSDM_k^{t^\top} \otimes \I_{N}\right)\vF_k^{t+1} - \vI_N + \bar{\boldsymbol{\alpha}}_k^t \, ;
\end{equation}
the update for $\vU_k^{t+1}$ in \eqref{eq:upF} reads as
\begin{equation}\label{eq:upU}
    \begin{aligned}
        \vU_k^{t+1}=&\argmin_{\vU_k} \dfrac{\rho}{2}\sum_k \norm{\mathbf{u}_k-\vU_k}^2+\sum_k \norm{\vU_k}_\infty\\ 
        =& \mathsf{prox}_{\tfrac{1}{\rho}\norm{\cdot}_\infty} \left( \mathbf{u}_k \right).
    \end{aligned}
\end{equation}
Hence, by resorting to \Cref{lem:complexprox}, we obtain
\begin{equation}\label{eq:upUprox}
    \vU_k^{t+1}=\mathrm{sign}(\mathbf{u})\left(\abs{\mathbf{u}} - \dfrac{1}{\rho} \Pi_{\mathcal{B}(\rho)}\left( \rho\abs{\mathbf{u}}\right)\right),
\end{equation}
where, similarly to \Cref{subsub:updateX}, $\Pi_{\mathcal{B}(\rho)}$ is the projection operator onto the $\ell_1$-norm ball of radius $\rho$.

\subsubsection{Solution for \texorpdfstring{$\vL_k^{t+1}$}{TEXT}}\label{subsub:updateL}
Analogously to \Cref{subsub:updateW}, we first derive the solution in matrix form, and then we vectorize it. From \eqref{eq:aulF}, letting 
\begin{equation}
    \mathbf{L}_k = \PSDM_k^{t+1}+\bar{\boldsymbol{\Delta}}_k^{t},
\end{equation}
with $\bar{\boldsymbol{\Delta}}_k^{t}$ being the $N \times N$ matrix representation of $\bar{\boldsymbol{\delta}}_k^{t}$, the update for $\Ll_k^{t+1}$ in \eqref{eq:upF} is
\begin{equation}\label{eq:probL}
    \begin{aligned}
        \Ll_k^{t+1} &= \argmin_{\Ll_k} \;\;\dfrac{\delta}{2}\norm{\mathbf{L}_k\!-\!\Ll_k}^2 \!\!+\! I_{\Se}(\Ll_k)\!=\\
        &= \argmin_{\Ll_k \in \Se} \;\;\norm{\mathbf{L}_k-\Ll_k}^2 \, .
    \end{aligned}
\end{equation}
Again, following \cite{parikh2014proximal}, the solution of \eqref{eq:probL} corresponds to projecting the Hermitian matrix $\bar{\mathbf{L}}_k = (\mathbf{L}_k +\mathbf{L}_k^\conjugateT)/2$ onto $\Se$, where the projection reads as 
\begin{equation}
\begin{aligned}
    &\mathring{\mathbf{L}}_k = \Pi_{\Se}(\bar{\mathbf{L}}_k) = \sum_{i\in [N]} \max(\bar{l}_i,\epsilon)\,\mathring{\mathbf{n}}_i\mathring{\mathbf{n}}^\conjugateT_i \, ,\\
    &\text{with} \quad \bar{\mathbf{L}}_k=\sum_i \bar{l}_i\,\mathring{\mathbf{n}}_i\mathring{\mathbf{n}}^\conjugateT_i \, .
\end{aligned}
\end{equation}
Finally, we set 
\begin{equation}\label{eq:upL}
    \vL_k^{t+1}=\mathrm{vec}(\mathring{\mathbf{L}}_k) \, .    
\end{equation}

\subsection{The algorithm}\label{subsec:algo}
To summarize the steps in \Cref{subsec:subproblemP,subsec:subproblemF}, by exploiting \Cref{lem:updateP,lem:complexprox} we solve all the updates, but $\vF_{k}^{t+1}$, in closed form. 
Specifically, \Cref{lem:updateP} provides the solution for $\vP_k^{t+1}$. 
Then, \Cref{lem:complexprox} is used to compute the proximal operators for $\vX_{k}^{t+1}$, $\vV_{\K_m}^{t+1}$, and $\vU_{k}^{t+1}$. 
The positive-definiteness of the solutions is enforced by projecting $\vW_{k}^{t+1}$ and $\vL_{k}^{t+1}$ onto $\vSe$. 
Finally, the updates for $\vF_{k}^{t+1}$ and $\beta_{k}^{t+1}$ are computed jointly in quasi-closed form by solving the corresponding KKT in \Cref{eq:kktF}.
All the steps of the proposed iterative procedure are detailed in \Cref{alg:nonconvex}. After the variables' initialization (cf. lines 3-18), the algorithm proceeds exploiting the ADMM recursions \eqref{eq:upP} and \eqref{eq:upF} derived in the previous paragraphs (cf. lines 20-35), while also applying a diminishing step-size rule on the updates of $\vP_k^{t+1}$ (cf. line 23) and $\vF_k^{t+1}$ (cf. line 28), as required by the SCA framework. To enable convergence with inexact updates, SCA requires a diminishing step-size rule $\xi^t$ satisfying classical stochastic approximation conditions \cite{scutari2016parallel}:
\[\xi^t\in(0,1]\,, \qquad \sum_t \xi^t=\infty\,, \qquad \sum_t (\xi^t)^2<\infty\,.\]
Specifically, here we use the stepsize sequence~\cite{scutari2016parallel}:
\begin{equation}\label{eq:stepsize}
    \xi^t=\dfrac{\xi^{t-1}+\xi_1(t)^{c_1}}{1+c_2\xi_2(t)}\,, \qquad \xi^0=1\,, \quad 0<c_1 \leq c_2<1\,;
\end{equation}
where $\xi_1(t)/\xi_2(t)\rightarrow 0$ as $t\rightarrow \infty$. For instance, in our experiments, we consider the pairs $(\xi_1(t),\xi_2(t))=(\log{t},t)$, and $(\xi_1(t),\xi_2(t))=(\log{t},\sqrt{t})$.

\begin{algorithm}[t]
    \small
   \caption{IA method.}
   \label{alg:nonconvex}
\begin{algorithmic}[1]
    \STATE {\bfseries Input:} $\inPSDT_{ijk}$, $\{k_1, k_2, \dots, k_K\}$, $\lambda$, $\tau$, $c_1$, $c_2$, \\$\tau_{abs}^p$, $\tau_{abs}^d$, $\tau_{rel}^p$, $\tau_{rel}^d$, $\epsilon$, $\mathrm{init}$ (of $\IPSDM_k$).
    \STATE {\bfseries Output:} $\{\PSDM_k\}$, $\{\IPSDM_k\}$.
   \STATE Initialize $t\gets0$, $\xi^0\gets1, \delta=\omega=\rho=\sigma=\theta \gets 1$,
   \FOR{$k=0$ {\bfseries to} $M-1$}
   \STATE Initialize $\PSDM_k \gets \inPSDM_k$, 
   \IF{$\mathrm{init}==\mathrm{identity}$}
   \STATE \quad $\IPSDM_k \gets \I_N,$
   \ELSIF{$\mathrm{init}==\mathrm{inverse}$}
   \STATE \quad $\IPSDM_k \gets \inPSDM_k^{-1},$
   \ENDIF
   \STATE Initialize $\vU_k \gets (\IPSDM_k^{t^\top} \otimes \I_{N})\vF_k - \vI_N$,
   \STATE Initialize $\vX_k \gets (\I_{N} \otimes \PSDM_k^t)\vP_k - \vI_N$,
   \STATE Initialize $\vL_k \gets \vF_k$,
   \STATE Initialize $\vW_k \gets \vP_k$,
   \STATE Initialize $\beta_k \gets 0$,
   \STATE Initialize $\bar{\boldsymbol{\alpha}}_k=\bar{\boldsymbol{\delta}}_k=\bar{\boldsymbol{\mu}}_k=\bar{\boldsymbol{\omega}}_k\gets \mathbf{0}+\mathbf{0}\iota \in \C^{N^2}$,
   \ENDFOR
   \STATE Initialize $\vV \gets \left(\SEL \otimes \I_{M}\right)\vP_{(\positivefrequency)}$, $\bar{\boldsymbol{\phi}}\gets \mathbf{0}+\mathbf{0}\iota \in \C^{N^2 M}$,
   \REPEAT
   \STATE \textbf{do in parallel}
   \STATE \quad \textbf{do in parallel} over $k \in \positivefrequency$
   \STATE \quad \quad $\vP_k^{t+1}$ $\gets$ \Cref{eq:updateP},
   \STATE \quad \quad $\vP_k^{t+1} \gets \vP_k^{t}+\xi^t\left(\vP_k^{t+1}-\vP_k^{t}\right)$, 
   \STATE \quad \quad $\vX_k^{t+1}$, $\vW_k^{t+1}$ $\gets$ \Cref{eq:upXprox,eq:upW},
   \STATE \quad \quad $\bar{\boldsymbol{\mu}}_k^{t+1}$, $\bar{\boldsymbol{\omega}}_k^{t+1}$ $\gets$ Steps (4-5) in \eqref{eq:upP},
   \STATE \quad \textbf{do in parallel} over $k \in \positivefrequency$
   \STATE \quad \quad $\vF_k^{t+1}$, $\beta_k^{t+1}$ $\gets$ \Cref{eq:upF_and_beta},
   \STATE \quad \quad $\vF_k^{t+1} \gets \vF_k^{t}+\xi^t\left(\vF_k^{t+1}-\vF_k^{t}\right)$,
   \STATE \quad \quad $\vU_k^{t+1}$, $\vL_k^{t+1}$  $\gets$ \Cref{eq:upUprox,eq:upL},
   \STATE \quad \quad $\bar{\boldsymbol{\alpha}}_k^{t+1}$, $\bar{\boldsymbol{\delta}}_k^{t+1}$ $\gets$ Steps (4-5) in \eqref{eq:upF},
   \FOR{$m=1$ {\bfseries to} $K$}
   \STATE \quad \textbf{do in parallel} over $j$
   \STATE \quad \quad $\left[\vV_{\K_m}^{t+1}\right]_j \gets $ \Cref{eq:upVprox},
   \STATE $\left[\bar{\boldsymbol{\phi}}^{t+1}\right]_{\K_m} \gets $ Step (7) in \eqref{eq:upP}
   \ENDFOR
   \STATE $t \gets t+1$
   \STATE $\xi^t \gets \dfrac{\xi^{t-1}+\xi_1(t)^{c_1}}{1+c_2\xi_2(t)}$
   \UNTIL residuals convergence
\end{algorithmic}
\end{algorithm}

\noindent\textbf{Stopping criteria.} 
The stopping criterion for the proposed algorithm follows from the primal and dual feasibility optimality conditions~\cite{boyd2011distributed}.
For each $k \in \positivefrequency$, the primal residuals, associated with constraints on the primal variables, read as:
$\mathring{\mathbf{p}}_{1,k}^{t+1} = (\I_{N} \otimes \PSDM_k^t)\vP_k^{t+1} - \vI_N - \vX_k^{t+1}$,
$\mathring{\mathbf{p}}_{2,k}^{t+1} = (\IPSDM_k^{t^\top} \otimes \I_{N})\vF_k^{t+1} - \vI_N - \vU_k^{t+1}$,
$\mathring{p}_{3,k}^{t+1} = \norm{\vF_{k}^{t+1} - \vFtilde_k}^{2} - \eta$,
$\mathring{\mathbf{p}}_{4,k}^{t+1} = \SEL\vP_k^{t+1} - \vV_k^{t+1}$,
$\mathring{\mathbf{p}}_{5,k}^{t+1} =\vP_k^{t+1}-\vW_k^{t+1}$,
$\mathring{\mathbf{p}}_{6,k}^{t+1} =\vF_k^{t+1}-\vL_k^{t+1}$;
from which we set $\boldsymbol{\Pi}_{1}^{t+1}=[\mathring{\mathbf{p}}_{1,0}^{t+1},\ldots,\mathring{\mathbf{p}}_{1,M-1}^{t+1}]$, $\boldsymbol{\Pi}_{2}^{t+1}=[\mathring{\mathbf{p}}_{2,0}^{t+1},\ldots,\mathring{\mathbf{p}}_{2,M-1}^{t+1}]$, $\boldsymbol{\Pi}_{4}^{t+1}=[\mathring{\mathbf{p}}_{4,0}^{t+1},\ldots,\mathring{\mathbf{p}}_{4,M-1}^{t+1}]$, $\boldsymbol{\Pi}_{5}^{t+1}=[\mathring{\mathbf{p}}_{5,0}^{t+1},\ldots,\mathring{\mathbf{p}}_{5,M-1}^{t+1}]$, $\boldsymbol{\Pi}_{6}^{t+1}=[\mathring{\mathbf{p}}_{6,0}^{t+1},\ldots,\mathring{\mathbf{p}}_{6,M-1}^{t+1}]$, belonging to $\C^{N^2 \times M}$; and $\mathbf{p}_{3}^{t+1}=[p_{3,0}^{t+1},\ldots,p_{3,M-1}^{t+1}] \in \R^{M}$.

The dual residuals can be obtained from the stationarity conditions. 
Specifically, $\forall k \in \positivefrequency$, given the point $\vZ_{k}^{t+1}\coloneqq(\vF_k^{t+1}, \vP_k^{t+1})$, from the stationarity condition it follows that  
\begin{equation}\label{eq:statcondP}
    0 \!\in\! \nabla_{\vP_k^*}{\mathcal{L}_{\sigma, \omega, \theta}(\vP_k ; \vX_k^t\!, \vW_k^t, \vV^t\!, \bar{\boldsymbol{\mu}}_k^t, \bar{\boldsymbol{\omega}}_k^t, \bar{\boldsymbol{\phi}}^t\!, \vZ_k^t)}\at[\bigg]{\vP_k^{t+1}}; 
\end{equation}
and
\begin{equation}\label{eq:statcondF}
    0 \in \nabla_{\vF_k^*}{\widetilde{\mathcal{L}}_{\rho, \delta}(\vF_k, \beta_k; \vU_k^t, \vL_k^t, \bar{\boldsymbol{\alpha}}_k^t, \bar{\boldsymbol{\delta}}_k^t, \vZ_k^t, \vFtilde_k)}\at[\bigg]{(\vF_k^{t+1},\beta_k^{t+1})}.
\end{equation}
From \eqref{eq:statcondP}, following the steps in \cite{boyd2011distributed}, we get
\begin{equation}\label{eq:AULdualP}
    \begin{aligned}
    0 \in & \tau(\vP_k^{t+1}-\vP_k^t) + \omega\bar{\boldsymbol{\omega}}_k^{t+1} + \sigma\left(\I_{N} \otimes \PSDM_k^t \right)^\conjugateT \bar{\boldsymbol{\mu}}_k^{t+1} +\\
    &+\theta\SEL^\top \bar{\boldsymbol{\phi}}_k^{t+1}+\omega(\vW_k^{t+1}-\vW_k^t) + \\
     & + \sigma(\I_{N} \otimes \PSDM_k^t)^\conjugateT \bigg(\vX_k^{t+1}-\vX_k^{t} \bigg) + \theta\SEL \left(\vV_k^{t+1} - \vV_k^{t}\right).
    \end{aligned}
\end{equation}
In our case, denoting with $(\vP_k^\star\!, \vX_k^\star\!, \vW_k^\star, \vV^\star\!, \bar{\boldsymbol{\mu}}_k^\star, \bar{\boldsymbol{\omega}}_k^\star, \bar{\boldsymbol{\phi}}^\star\!)$ a saddle point for the unaugmented Lagrangian for the subproblem in $\vP_k$, the dual feasibility reads as
\begin{equation}\label{eq:dualfeasP}
    \begin{aligned}
        0 \in& \tau(\vP_k^\star-\vP_k^t) \!+\! \omega\bar{\boldsymbol{\omega}}_k^\star \!+\! \sigma\left(\I_{N} \!\otimes\! \PSDM_k^t \right)^\conjugateT\!\! \bar{\boldsymbol{\mu}}_k^\star \!+\!\theta\SEL^\top \bar{\boldsymbol{\phi}}_k^\star\,.
    \end{aligned}
\end{equation}
Hence, from \eqref{eq:AULdualP}, we have that the quantity 
\begin{equation}\label{eq:dualresidP}
    \begin{aligned}
        \mathbf{d}_{1,k}^{t+1}&=\!\omega(\vW_k^{t+1}-\vW_k^t)+\sigma(\I_{N} \otimes \PSDM_k^t)^\conjugateT \!\bigg(\vX_k^{t+1}\!-\!\vX_k^{t} \bigg)\!+\\
        &+\theta\SEL \left(\vV_k^{t+1}\!-\!\vV_k^{t}\right),
    \end{aligned}
\end{equation}
can be considered as a residual for the dual feasibility condition in \eqref{eq:dualfeasP}.
By applying the same steps for the subproblem in $\vF_k$, starting from \eqref{eq:statcondF}, we get the formula for the other residual,
\begin{equation}\label{eq:dualresidF}
    \mathbf{d}_{2,k}^{t+1}=\dfrac{\delta}{2}(\vL_k^{t+1}\!\!-\vL_k^t)+\dfrac{\rho}{2}\left(\IPSDM_k^{t^\top}\! \otimes \I_{N}\right)^\conjugateT \!\! \bigg(\vU_k^{t+1}\!\!-\vU_k^{t} \bigg)\,.
\end{equation}
Since the expressions in \eqref{eq:dualresidP} and \eqref{eq:dualresidF} are specific for frequency $k$, $k \in \positivefrequency$, also in this case we set $\boldsymbol{\Delta}_{1}^{t+1}=[\mathbf{d}_{1,0}^{t+1},\ldots,\mathbf{d}_{1,M-1}^{t+1}]$, $\boldsymbol{\Delta}_{2}^{t+1}=[\mathbf{d}_{2,0}^{t+1},\ldots,\mathbf{d}_{2,M-1}^{t+1}]$, both in $\C^{N^2 \times M}$.

As $t\rightarrow \infty$, the norms of primal and dual residuals above should vanish.
According to \cite{boyd2011distributed}, the stopping criterion is met when the norms of those residuals are below certain tolerances, constituted by an absolute and a relative term.
Let us set 
$\mathbf{T}_1\!\!\!=\!\!\!\Big[(\I_{N} \!\otimes\! \PSDM_0^t)\vP_0^{t+1}, \ldots, (\I_{N} \!\otimes\! \PSDM_{M-1}^t)\vP_{M-1}^{t+1}\Big]$,
$\mathbf{T}_2\!\!\!=\!\!\!\Big[\vX_0^{t+1}, \ldots, \vX_{M-1}^{t+1}\Big]$,
$\mathbf{T}_3\!\!\!=\!\!\!\Big[(\IPSDM_0^{t^\top} \otimes \I_{N})\vF_0^{t+1}, \ldots, (\IPSDM_{M-1}^{t^\top} \otimes \I_{N})\vF_{M-1}^{t+1}\Big]$,
$\mathbf{T}_4=\Big[\vU_0^{t+1}, \ldots, \vU_{M-1}^{t+1}\Big]$,
$\mathbf{t}_5\!\!=\!\!\Big[\norm{\vF_0^{t+1}\!\!-\vFtilde_0}, \ldots, \norm{\vF_{M-1}^{t+1}-\vFtilde_{M-1}}\Big]$,
$\mathbf{T}_6\!\!\!=\!\!\!\Big[\SEL\vP_0^{t+1},$$ \ldots, \SEL\vP_{M-1}^{t+1}\Big]$,
$\mathbf{T}_7\!\!\!=\!\!\!\Big[\vV_0^{t+1}, \ldots, \vV_{M-1}^{t+1}\Big]$,
$\mathbf{T}_8\!\!\!=\!\!\!\Big[\vP_0^{t+1}, \ldots, \vP_{M-1}^{t+1}\Big]$,
$\mathbf{T}_9\!\!\!=\!\!\!\Big[\vW_0^{t+1}, \ldots, \vW_{M-1}^{t+1}\Big]$,
$\mathbf{T}_{10}\!\!\!=\!\!\!\Big[\vF_0^{t+1}, \ldots, \vF_{M-1}^{t+1}\Big]$,
$\mathbf{T}_{11}\!\!\!=\!\!\!\Big[\vL_0^{t+1}, \ldots, \vL_{M-1}^{t+1}\Big]$,
$\mathbf{T}_{12}\!\!\!=\!\!\!\Big[\bar{\boldsymbol{\omega}}_0^{t+1}, \ldots, \bar{\boldsymbol{\omega}}_{M-1}^{t+1}\Big]$,
$\mathbf{T}_{13}\!\!\!=\!\!\!\Big[\left(\I_{N} \otimes \PSDM_0^t \right)^\conjugateT \bar{\boldsymbol{\mu}}_0^{t+1}, \ldots, \left(\I_{N} \otimes \PSDM_{M-1}^t \right)^\conjugateT \bar{\boldsymbol{\mu}}_{M-1}^{t+1}\Big]$,
$\mathbf{T}_{14}\!\!\!=\!\!\!\Big[\SEL^\top \bar{\boldsymbol{\phi}}_0^{t+1}, \ldots, \SEL^\top \bar{\boldsymbol{\phi}}_{M-1}^{t+1}\Big]$,
$\mathbf{T}_{15}\!\!\!=\!\!\!\Big[\bar{\boldsymbol{\delta}}_0^{t+1}, \ldots, \bar{\boldsymbol{\delta}}_{M-1}^{t+1}\Big]$,
$\mathbf{T}_{16}\!\!\!=\!\!\!\Big[\left(\IPSDM_0^{t^\top} \otimes \I_{N}\right)^\conjugateT \bar{\boldsymbol{\alpha}}_0^{t+1}, \ldots, \left(\IPSDM_{M-1}^{t^\top} \otimes \I_{N}\right)^\conjugateT \bar{\boldsymbol{\alpha}}_{M-1}^{t+1}  \Big]$.
Hence, we define the stopping rules as composed of an absolute and a relative component.
In detail, given absolute and relative tolerances, namely (i) $\tau_{abs}^p$, $\tau_{rel}^p$ for primal residuals, and (ii) $\tau_{abs}^d$,  $\tau_{rel}^d$ for dual ones; the stopping rules read as
\begin{equation*}\label{eq:convergence}
    \begin{aligned}
        \norm{\boldsymbol{\Pi}_{1}^{t+1}} &\!<\!N\sqrt{M} \tau_{abs}^p \!+\! \tau_{rel}^p \max\{\norm{\mathbf{T}_1}, \norm{\mathbf{T}_2}, \sqrt{MN}\},\\
        \norm{\boldsymbol{\Pi}_{2}^{t+1}} &\!<\!N\sqrt{M} \tau_{abs}^p \!+\! \tau_{rel}^p \max\{\norm{\mathbf{T}_3}, \norm{\mathbf{T}_4}, \sqrt{MN}\},\\
        \norm{\mathbf{p}_{3}^{t+1}} &\!<\!N\sqrt{M} \tau_{abs}^p \!+\! \tau_{rel}^p \max\{\norm{\mathbf{t}_5},\eta\sqrt{M}\},\\
        \norm{\boldsymbol{\Pi}_{4}^{t+1}} &\!<\!N\sqrt{M} \tau_{abs}^p \!+\! \tau_{rel}^p \max\{\norm{\mathbf{T}_6}, \norm{\mathbf{T}_7}\},\\
        \norm{\boldsymbol{\Pi}_{5}^{t+1}} &\!<\!N\sqrt{M} \tau_{abs}^p \!+\! \tau_{rel}^p \max\{\norm{\mathbf{T}_8}, \norm{\mathbf{T}_9}\},\\
        \norm{\boldsymbol{\Pi}_{6}^{t+1}} &\!<\!N\sqrt{M} \tau_{abs}^p \!+\! \tau_{rel}^p \max\{\norm{\mathbf{T}_{10}}, \norm{\mathbf{T}_{11}}\},\\
        \norm{\boldsymbol{\Delta}_{1}^{t+1}} &\!<\!N\sqrt{M} \tau_{abs}^d \!+\! \tau_{rel}^d \max\{\omega\norm{\mathbf{T}_{12}},\! \sigma\norm{\mathbf{T}_{13}},\! \theta\norm{\mathbf{T}_{14}}\!\},\\
        \norm{\boldsymbol{\Delta}_{2}^{t+1}} &\!<\!N\sqrt{M} \tau_{abs}^d \!+\! \dfrac{1}{2}\tau_{rel}^d \max\{\delta\norm{\mathbf{T}_{15}}, \rho\norm{\mathbf{T}_{16}}\}.
    \end{aligned}
\end{equation*}
With regards to the values for the tolerances, they are usually set to be small, for instance $10^{-3}$.

\noindent\textbf{Computational cost.} 
The updates of the subproblems in \Cref{subsec:subproblemP,subsec:subproblemF} can be run in parallel. In addition, each subproblem is further parallelizable over $k \in \positivefrequency$.
With reference to \Cref{alg:nonconvex}, computing the update in line 22 involves the inversion of $\boldsymbol{\Gamma}_1$ in \eqref{eq:Gamma1}, i.e., $N$ inversions of $N\times N$ matrices, costing $\mathcal{O}(N^4)$ flops. 
Additionally, we have block-diagonal matrix-vector multiplications, requiring $\mathcal{O}(N^3)$ flops.
Afterwards, line 23 requires $\mathcal{O}(N^2)$ operations.
Hence, overall we need $\mathcal{O}(N^4)$ flops for the update of $\vP_k^{t+1}$ in lines 22-23.
Concerning line 24, the update of $\vX_k$ in \eqref{eq:upXprox} requires the computation of $\mathbf{x}_k$ in \eqref{eq:x}, costing $\mathcal{O}(N^3)$ flops, and the projection onto the $\ell_1$-norm ball, which is linear in the number of terms and thus costs $\mathcal{O}(N^2)$.
Therefore, the cost is $\mathcal{O}(N^3)$.
Furthermore, the update of $\vW_k$ given in \eqref{eq:upW} involves the eigedecomposition of an $N\times N$ matrix, which costs $\mathcal{O}(N^3)$.
Next, in line 25, the updates of $\bar{\boldsymbol{\mu}}_k$ and $\bar{\boldsymbol{\omega}}_k$ are linear in the number of terms, costing $\mathcal{O}(N^2)$ flops each.  
Therefore, the block of lines 21-25 requires asymptotically $\mathcal{O}(N^4)$ flops.

Subsequently, the updates for $\vF_k$ and $\beta_k$ in line 27 involve the solution of the KKT conditions in \eqref{eq:kktF} via the bisection method.
Then, the cost depends on the number of bisection iterations, where each evaluation of \eqref{eq:updateF} involves the inversion of $\boldsymbol{\Gamma}_2$ in \eqref{eq:Gamma2}, which costs $\mathcal{O}(N^3)$ since $\boldsymbol{\Gamma}_2$ is block-diagonal with equal blocks of size $N \times N$.
Additionally, we have block-diagonal matrix-vector multiplications, requiring $\mathcal{O}(N^3)$ flops as above.
Hence, the overall cost for evaluating \eqref{eq:updateF} is $\mathcal{O}(N^3)$.
At this point, the smoothing in line 28 requires $\mathcal{O}(N^2)$ operations.
Next, with regards to line 29, the update of $\vU_k$ in \eqref{eq:upUprox} requires computing $\mathbf{u}$ in \eqref{eq:u}, which costs $\mathcal{O}(N^3)$ flops, and subsequently the projection onto the $\ell_1$-norm ball, which again costs $\mathcal{O}(N^2)$.
Besides, the update of $\vL_k$ given in \eqref{eq:upL} involves the eigedecomposition of an $N\times N$ matrix, which requires $\mathcal{O}(N^3)$ operations.
Then, the updates of $\bar{\boldsymbol{\mu}}_k$ and $\bar{\boldsymbol{\omega}}_k$ in line 30 require $\mathcal{O}(N^2)$ flops each.
Assuming that the number of bisection iterations is small, the block of lines 26-30 asymptotically costs $\mathcal{O}(N^3)$ flops. Next, in line 33, the update of $[\vV_{\K_m}]_j$ in \eqref{eq:upVprox} involves the computation of $[\mathbf{v}_{\K_m}]_j$ in \eqref{eq:v}. Now, since $\left(\I_{|\K_m|}\otimes \SEL \right)$ is diagonal, the latter computation costs $\mathcal{O}(|\K_m|N^2)$ flops. Finally, the update of $\left[\bar{\boldsymbol{\phi}}\right]_{\K_m}$ in line 34 is linear in the number of terms, thus costing $\mathcal{O}(|\K_m|N^2)$.
Hence, for each $m \!\in\! [K]$, lines 32-34 have an asymptotic cost of $\mathcal{O}(|\K_m|N^2)$ flops.

In summary, asymptotically the computational cost of our algorithm is dominated by the updates of $\vP_{k}^{t+1}$ and $\vF_{k}^{t+1}$, that we run in parallel.
In case the number of bisection steps $N_{b}$ is lower than $N$, then the most demanding update is that of $\vP_{k}^{t+1}$, costing $\mathcal{O}(N^4)$ flops.
Vice versa, the most demanding will be that of $\vF_{k}^{t+1}$, costing $\mathcal{O}(N_b N^3)$ flops. 

\section{Experiments on synthetic data}\label{sec:synth_exp}

\begin{figure}[htbp]
    \centering
    \includegraphics[width=1\columnwidth]{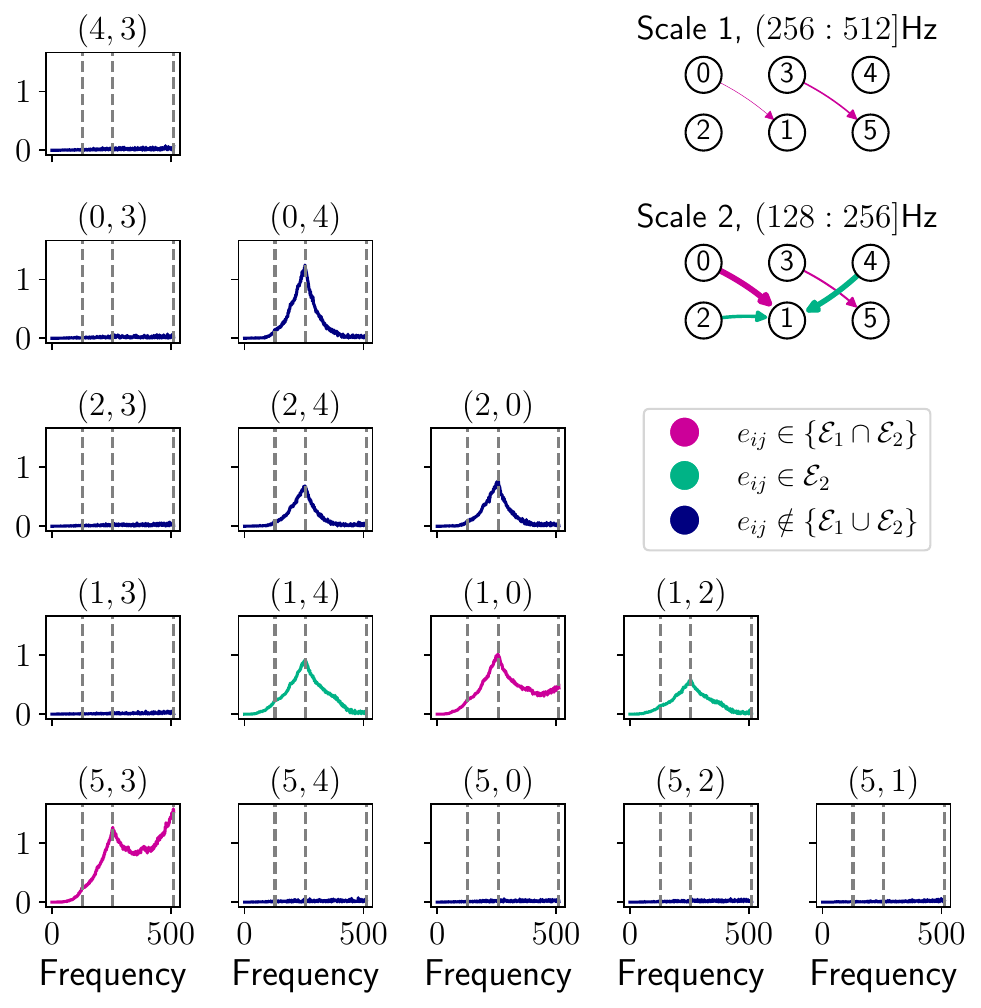}
    \caption{The figure depicts the underlying multiscale causal structure together with the behavior of the strictly lower triangular part of $\mathbf{P}_k$ along frequencies. Different colors are associated with interactions that (i) are not present in the causal structure (blue), (ii) exist within the causal graph related to time scale 2 (green), and (iii) show up in both time scales (magenta). Finally, dashed lines indicate the splitting points of the frequency bands associated with the relevant time scales.}
    \label{fig:P_mdag}
\end{figure}

In this section, we assess and compare the performance of the proposed methods on synthetic data. To generate time series data entailing a block-sparse inverse CSD, we exploit the multiscale linear causal model proposed in~\cite{d'acunto2023learning}. Specifically, we generate $N_{\mathcal{Y}}=10000$ data sets consisting of $N=6$ time series of length $T=1024$ from a multiscale causal structure, where stationary interactions among time series occur only at two time scales, corresponding to the frequency bands $(256:512]$Hz and $(128:256]$Hz, respectively.
Then, we use those $N_{\mathcal{Y}}$ data sets to obtain an accurate estimate of the inverse CSD tensor, from which we retrieve the $K$-PCG where w.l.o.g. we set $K=8$ with equally-sized frequency blocks within the interval $[0:512]$Hz. The cardinalities of the $8$-PCG arc sets are $\{|\mathcal{E}_m|\}_{m \in [K]}=\{0,5,7,7,7,6,4,2\}$.

\Cref{fig:P_mdag} depicts the multiscale causal structure together with the behavior of the strictly lower triangular part of $\mathbf{P}_k$ along frequencies, where we use dashed lines to locate the splitting points of the relevant frequency bands.
Since the inverse CSD tensor concerns partial correlation, we observe additional interactions that are not defined in the underlying causal structure, and which can be understood in light of the \emph{d-separation criterion}~\cite{pearl2009causality}. For instance, the interaction between node $0$ and $4$ follows from the fact that, when we condition on the set of vertices $\mathcal{V}_{0,4}=\{1,2,3,5\}$, the latter also contains node $1$ that is a child of both $0$ and $4$ within the causal structure corresponding to the second time scale. Subsequently, the path $0\rightarrow 1 \leftarrow 4$ activates, and we observe partial correlation.

We compare our proposed methods with different baselines, including the naive estimator $\inIPSDT$ (cf. Sec. \ref{sec:preliminaries}) and the \emph{time series graphical LASSO} algorithm in~\cite{tugnait2022sparse} (TS-GLASSO), which combines the $\ell_{1}$- and $\ell_{21}$-norm regularizations via $\alpha \in [0,1]$ (cf. eq. (41) in \cite{tugnait2022sparse}). Here we test $\alpha=0$ and $\alpha=0.5$. We consider two versions of the CF method: \emph{CF-nz} with $s_m=s=7$, where the user knows only the maximum number of non-null dependencies, and \emph{CF-fk} with $s_m=|\mathcal{E}_m|, \forall m$, where the full knowledge of $\{|\mathcal{E}_m|\}_{m \in [K]}$ is provided. In addition, we evaluate two versions of the IA method: \emph{IA-gs}, which applies $\ell_{21}$-norm regularization without considering distinct frequency blocks (similarly to TS-GLASSO), and \emph{IA-bs}, which divides the frequency domain into $K=3$ blocks with starting frequencies $\{k_m\}=\{0, 64, 448\}$Hz to enforce block-sparsity.
To assess the quality of the learned graphs, we use the \emph{structural Hamming distance} (SHD), which quantifies the number of modifications required to convert the estimated graph into the ground truth graph. The regularization strength parameter $\lambda$ is fine-tuned within the range $[0.001,1]$ based on SHD, for both TS-GLASSO and IA. Furthermore, we set $\eta = 0.01$ in (\ref{eq:prob1}), and we test both strategies in lines 6-10 of \Cref{alg:nonconvex} for the initialization of $\IPSDM_k$. 
Convergence is determined using $\tau_{abs}^p=\tau_{rel}^p=\tau_{abs}^d=\tau_{rel}^d=5 \times 10^{-4}$ for both TS-GLASSO and IA. The learning process of these algorithms is stopped at $2000$ iterations to manage computational time, regardless of convergence conditions.
Convergence plots concerning the tested versions of our IA method are given in the \Cref{subsec:algo}, while the hyper-parameters values used in our simulations are listed in \Cref{tab:hyperparams}.

\begin{figure*}[t!]
    \centering
    \includegraphics[width=1\textwidth]{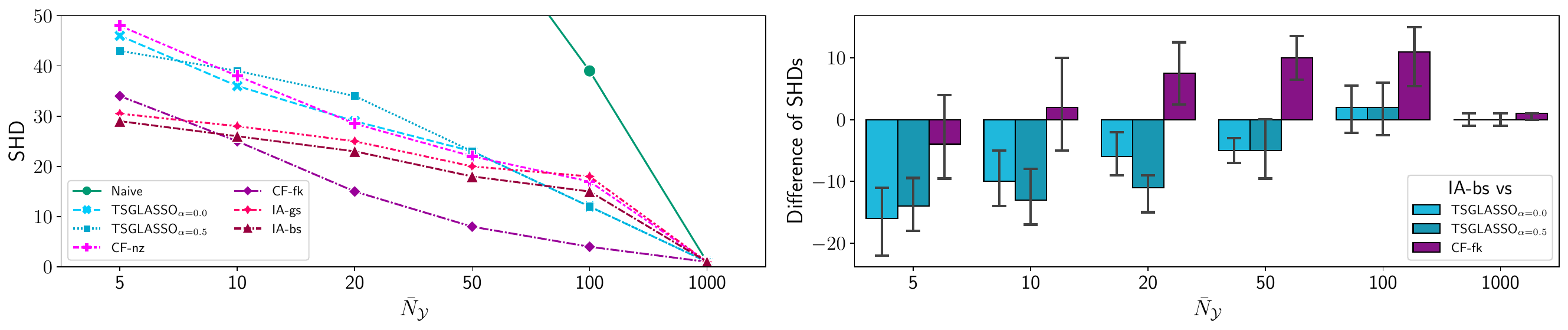}
    \caption{Left: performance in terms of SHD (lower better) for all methods, along $\bar{N}_{\mathcal{Y}}$.
    Right: difference of SHD between IA-bs and TS-GLASSO baseline, and CF-fk.
    Points (line plot) and bars height (bar plot) represent median values obtained over $50$ runs.
    Error bars (bar plot) represent $95\%$ interval.
    For readability, the line plot on the left has been cut at SHD$=50$, while in the bar plot on the right, we omit the naive baseline.}
    \label{fig:exp_synth}
    \vspace{-15pt}
\end{figure*}

We compare the performance of the methods in three settings with varying samples availability: \emph{scarce} ($\bar{N}_{\mathcal{Y}}=5$), \emph{medium} ($\bar{N}_{\mathcal{Y}} \in \{10, 20, 50\}$), and \emph{large} ($\bar{N}_{\mathcal{Y}} \in \{100, 1000\}$). For each setting, we estimate the smoothed periodogram using $\bar{N}_{\mathcal{Y}}$ data sets sampled from the same causal structure. Once obtained $\IPSDT$ from the learning methods, we compute the coherence $\RM_{k}=-\Dhat_k\IPSDM_{k}\Dhat_k$ for all $k \in \positivefrequency$, where $\Dhat_{k}$ is a diagonal matrix built as $\Dtilde_k$. The $8$-PCG is finally obtained by normalizing $\norm{\RM_{(\K_m)}}_2$ between $0$ and $1$ and keeping entries greater than $\bar{r}=0.05$. We repeat this procedure $50$ times for each setting using different data sets sampled from the same causal structure.

\Cref{fig:exp_synth} shows the SHD (left) and the differences in SHD (right) between IA-bs and TS-GLASSO, as well as CF-fk, with respect to $\bar{N}_{\mathcal{Y}}$. The line plot on the left is cut at SHD$=50$ for readability since the naive baseline shows SHD$>80$ for $\bar{N}_{\mathcal{Y}}\in \{5,10,20\}$. In the bar plot on the right, we omit the naive baseline for the same reason.
As we can see from the line plot in \Cref{fig:exp_synth} (left), all methods outperform the naive baseline, retrieving the ground truth when $\bar{N}_{\mathcal{Y}}=1000$. CF-fk, which has complete knowledge of $\{|\mathcal{E}_m|\}_{m \in [K]}$, performs the best overall, followed by IA-bs and IA-gs.
The superior performance of IA-gs over TS-GLASSO, despite not using block-sparsity, is due to the fact that it can deviate from the smoothed periodogram.
As $\bar{N}_{\mathcal{Y}}$ increases, the gap between the IA versions and the TS-GLASSO baselines reduces, while CF-fk keeps showing considerably lower SHD. The bar plot on the right of \Cref{fig:exp_synth} highlights the benefits of using block-sparsity.
Specifically, IA-bs significantly outperforms TS-GLASSO baselines for small-medium $\bar{N}_{\mathcal{Y}}$ values.
Then, in the large sample cases, differences vanish and are not statistically significant at the $95\%$ level.
In addition, despite a considerable disadvantage, IA-bs does not significantly differ from CF-fk for $\bar{N}_{\mathcal{Y}} \in \{5,10\}$, according to error bars. This is due to the low accuracy of the estimate of the inverse smoothed periodogram. Indeed, CF-fk receives the exact number of arcs per frequency band, and its errors are determined by the quality of the input estimate of the inverse smoothed periodogram. Thus, as the estimation improves, the performance gap between the two methods widens and then vanishes at $\bar{N}_{\mathcal{Y}}=1000$.

\subsection{Hyper-parameters}\label{subsec:hyper}
\Cref{tab:hyperparams} reports the list of all the hyper-parameters used by the IA and TS-GLASSO algorithms to obtain the numerical results depicted in \Cref{fig:exp_synth}.

\begin{table}[t]
    \centering
    \caption{Hyper-parameters used in the experimental assessment on synthetic data.}
    \resizebox{1.\columnwidth}{!}{%
    \begin{tabular}{rlrrrrrr}
        \toprule
        $\bar{N}_{\mathcal{Y}}$ & Method & $\lambda$ &init &$\xi_1(t)$&$\xi_2(t)$& $c_1$&$c_2$\\
        \midrule
        5    & IA-bs & 0.500&identity & $\log{t}$& $t$& 0.50&0.99\\
             & IA-gs & 0.700&identity & $\log{t}$& $t$& 0.50&0.99\\
             & $\text{TS-GLASSO}_{\alpha=0.0}$ & 1.000 &identity& N/A   & N/A   & N/A   &N/A   \\
             & $\text{TS-GLASSO}_{\alpha=0.5}$ & 0.700 &identity& N/A   & N/A   & N/A   &N/A   \\
             \midrule
        10   & IA-bs & 0.500&identity & $\log{t}$& $t$& 0.50&0.99\\
             & IA-gs & 0.500&identity & $\log{t}$& $t$& 0.50&0.99\\
             & $\text{TS-GLASSO}_{\alpha=0.0}$ & 1.000 &identity& N/A   & N/A   & N/A   &N/A   \\
             & $\text{TS-GLASSO}_{\alpha=0.5}$ & 0.500 &identity& N/A   & N/A   & N/A   &N/A   \\
             \midrule
        20   & IA-bs & 0.300&identity & $\log{t}$& $t$& 0.50&0.99\\
             & IA-gs & 0.300&identity & $\log{t}$& $t$& 0.50&0.99\\
             & $\text{TS-GLASSO}_{\alpha=0.0}$ & 1.000 &identity& N/A   & N/A   & N/A   &N/A   \\
             & $\text{TS-GLASSO}_{\alpha=0.5}$ & 0.300 &identity& N/A   & N/A   & N/A   &N/A   \\
             \midrule
        50   & IA-bs & 0.500&identity & $\log{t}$& $t$& 0.99&0.99\\
             & IA-gs & 0.500&identity & $\log{t}$& $t$& 0.99&0.99\\
             & $\text{TS-GLASSO}_{\alpha=0.0}$ & 1.000 &identity& N/A   & N/A   & N/A   &N/A   \\
             & $\text{TS-GLASSO}_{\alpha=0.5}$ & 0.070 &identity& N/A   & N/A   & N/A   &N/A   \\
             \midrule
        100  & IA-bs & 0.300&inverse & $\log{t}$& $t$& 0.90&0.99\\
             & IA-gs & 0.500&identity& $\log{t}$& $t$& 0.99&0.99\\
             & $\text{TS-GLASSO}_{\alpha=0.0}$ & 0.100 &identity& N/A   & N/A   & N/A   &N/A   \\
             & $\text{TS-GLASSO}_{\alpha=0.5}$ & 0.030 &identity& N/A   & N/A   & N/A   &N/A   \\
             \midrule
        1000 & IA-bs & 0.010 &inverse & $\log{t}$& $\sqrt{t}$& 0.50&0.99\\
             & IA-gs & 0.050&inverse & $\log{t}$& $\sqrt{t}$& 0.50&0.99\\
             & $\text{TS-GLASSO}_{\alpha=0.0}$ & 0.010 &identity& N/A   & N/A   & N/A   &N/A   \\
             & $\text{TS-GLASSO}_{\alpha=0.5}$ & 0.001 &identity& N/A   & N/A   & N/A   &N/A   \\
        \bottomrule
        \end{tabular}}
    \label{tab:hyperparams}
\end{table}

Looking at the behavior of $\lambda$ along $\bar{N}_{\mathcal{Y}}$, we see that as the number of available samples increases, $\lambda$ reduces for the considered methods. This is consistent with the increase in information that can be exploited by the algorithms.
Regarding the initialization of $\{\IPSDM_k\}$, we see that in the case of scarce and medium sample availability, the identity matrix is preferred by the IA method. Instead, in the case of large sample availability, the best initialization is the inverse of the smoothed periodogram. This is consistent with the fact that the quality of the smoothed periodogram increases with $\bar{N}_{\mathcal{Y}}$.
Concerning the stepsize rules hyper-parameters (i.e., $\xi_1(t)$, $\xi_2(t)$, $c_1$, $c_2$), we see that for a fixed number of maximum iterations for reaching convergence, as $\bar{N}_{\mathcal{Y}}$ increases, we can use slower decaying rules for minimizing the SHD. This means that in the scarce sample scenario, the IA method takes large steps at the beginning, but then needs to quickly reduce the stepsize to achieve better solutions. Instead, in the opposite scenario, due to the improvement in the estimation of the smoothed periodogram, the method can proceed more quickly in updating the optimization variables, and make the SHD vanish.
\section{Practical applications}\label{sec:rwdata}

The versatile nature of the proposed methods makes them relevant for a wide range of application domains. In this section, we outline practical scenarios in which our methods are suitable, but not limited to.

\begin{figure*}[tb]
    \centering
    \begin{subfigure}[b]{\textwidth}
        \centering
        \includegraphics[width=\textwidth]{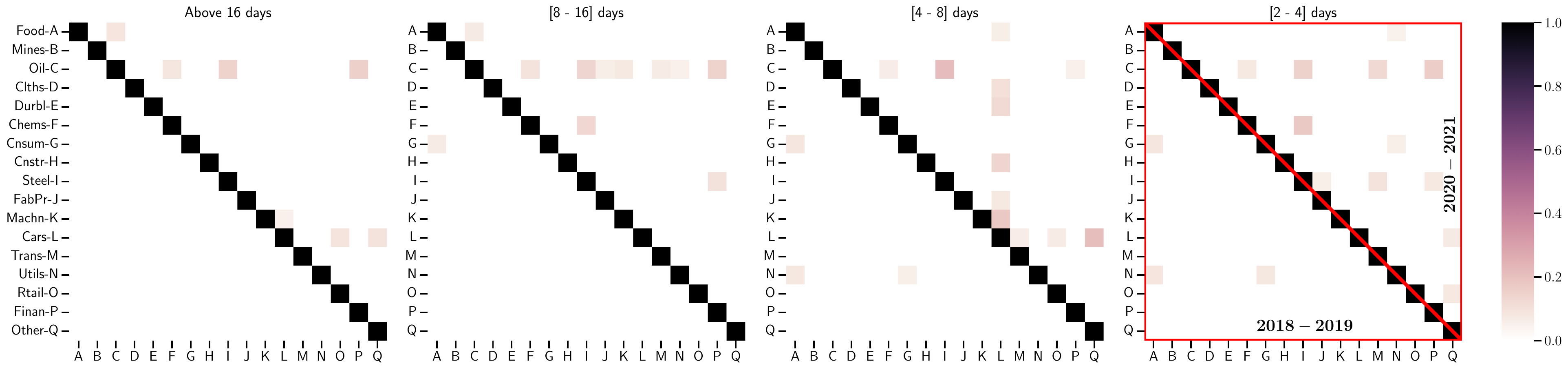}
        \caption{IA method.}
        \label{subfig:exp_fin_ours}
    \end{subfigure}
    \vfill
    \begin{subfigure}[b]{\textwidth}
        \centering
        \includegraphics[width=\textwidth]{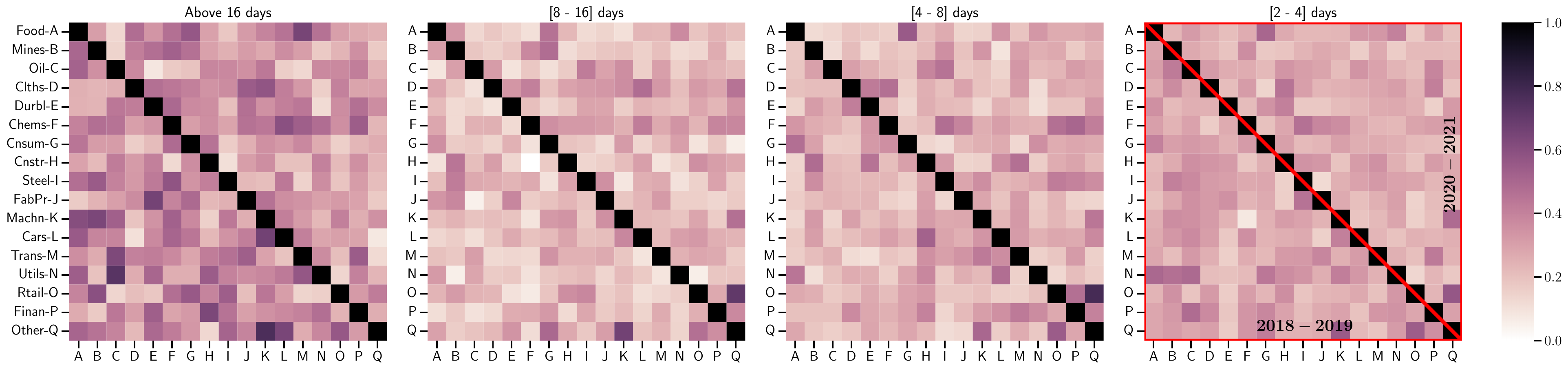}
        \caption{Naive method.}
        \label{subfig:exp_fin_naive}
    \end{subfigure}
    \caption{The 4 layers of the block-sparse multi-frequency partial correlation graph ($4$-PCG) among the time series of $17$ industry portfolios, where each layer corresponds to a different frequency band, provided by \subref{subfig:exp_fin_ours}) the IA method and \subref{subfig:exp_fin_naive}) the naive baseline. The lower triangular part of each matrix reports the $2018-2019$ period, while the upper triangular part the $2020-2021$ period.}
    \label{fig:exp_fin}
\end{figure*}

\subsection{Analysis of financial portfolios}\label{subsec:findata}
We study the partial correlations among daily returns of $17$ US industry equity portfolios from $5$ January 2018 to $31$ December 2021.
Daily returns were gathered from the \href{http://mba.tuck.dartmouth.edu/pages/faculty/ken.french/}{Fama-French repository}.
These portfolios are made by equities traded on the NYSE, AMEX, and NASDAQ stock exchanges.
Before applying the IA method, we make the time series of linear returns zero-mean. Hence, we divide the data set into two parts, $2018$-$2019$ and $2020$-$2021$, consisting of $T_1 \in \N$ and $T_2 \in \N$ days, respectively.
At this point, we estimate the smoothed periodogram corresponding to each period.
Here, we use a Hanning window with half-window size $\varsigma=\lfloor \sqrt{T_{\mathsf{min}}} \rfloor$, where $T_{\mathsf{min}} = \textrm{min}(T_1,T_2)$.
Such a choice is consistent with \cite{bohm2009shrinkage,sun2018large}, however, setting $\varsigma$ in a data-driven would also be possible \cite{ombao2001simple}.
The complete list of used hyper-parameters for the IA method is provided in the IPython notebook named \texttt{2.0\_real-world\_case\_study.ipynb}, accessible from our \href{https://github.com/OfficiallyDAC/BSPCG}{repository}.

We split the period into two parts, $2018$-$2019$ and $2020$-$2021$. The first period is characterized by strong economic growth, low unemployment, and moderate inflation.
However, trade tensions and geopolitical uncertainties were notable factors that influenced economic conditions during this period. The second period is characterized by a worsening macroeconomic environment, the outbreak and subsequent spread of the Covid-19 pandemic. Given the absence of prior knowledge, we apply the IA method to retrieve the $4$-PCG associated with $K=4$ frequency bands, capturing interactions within $[2-4]$, $[4-8]$, $[8-16]$, and above $16$ days. For example, the $[2-4]$ days block pertains to partial correlations occurring at a time resolution of $2$ consecutive days, up to a time resolution of $4$ days (i.e., roughly a business week). 

\Cref{subfig:exp_fin_ours} illustrates the 4 layers of the block-sparse multi-frequency partial correlation graph, learned by the IA method, where each layer corresponds to a different frequency band. From \Cref{subfig:exp_fin_ours}, we notice that over the first period (depicted in the lower triangular part of the matrices), few partial correlations occur, mainly at high-frequency bands. These dependencies are weak, and involve the food production industry (Food), businesses and industries focused on both essential and non-essential products (Cnsum), and energy services (Utils). These industries refer to segments of the economy that are relatively resistant to economic downturns and tend to remain stable during periods of economic recession or market volatility. Consequently, they are considered by investors to be defensive, meaning that they can provide stability during periods of economic uncertainty. On the contrary, during the second period (depicted in the upper triangular part of the matrices), partial correlations among portfolios are denser and spread over multiple frequency bands. These dependencies mainly relate to portfolios of energy resources (Oil); primary metal, iron and steel industries (Steel); and the automotive sector (Cars). In addition, while dependencies involving Oil occur over all frequency bands, those relating to Cars and Steel are localized within the $[4-8]$ and $[2-4]$ days time resolutions, respectively.

Our findings are justified by the economic conjuncture of the analyzed periods: the first characterized by a growing economy supported by fiscal policy, but also by the US-China trade war and increasing geopolitical uncertainties; the second by a worsening macroeconomic environment and the spread of Covid-19. In particular, the partial correlations learned over $2018$-$2019$ are the signs of trade tensions, economic uncertainty, and concerns about global economic growth.
These factors led investors to turn to defensive industry sectors, generating dependencies between Food, Cnsum, and Util industry portfolios.
Furthermore, during the first part of 2020, the oil market crashed~\cite{gharib2021impact}, and the natural gas market experienced its largest recorded demand decline in the history~\cite{iea2020}. Afterward, an upward trend started, always in a scenario of economic uncertainty. Thus, the partial correlations learned over $2020$-$2021$ serve as a clear indicator of the economic recession's progression, which has been further hastened by the widespread impact of the Covid-19 pandemic.

Finally, for comparison purposes, \Cref{subfig:exp_fin_naive} shows that the solution obtained through the naive estimator is notably denser than the one learned by our IA method. This observation aligns with the findings presented in \Cref{sec:synth_exp} regarding synthetic data. In scenarios with limited samples, the naive baseline consistently returns the densest solutions and exhibits the worst SHD values. Consequently, the solution provided by the naive estimator lacks valuable insights for effectively discriminating partial correlations among time series in different market scenarios and frequency bands. 

\subsection{Model-plant mismatch detection for industries}\label{subsec:MPMdata}
Model predictive control (MPC, \cite{garcia1989model, kouvaritakis2016model}) is a control strategy widely used in industrial processes, where a predictive model of the system (the plant) is employed to optimize control actions over a finite prediction horizon. However, in practice, there can be discrepancies between the predictive model and the actual plant behavior, known as model-plant mismatch (MPM).
The analysis of partial correlations between the model residuals and inputs is promising for the detection and diagnosis of MPM \cite{badwe2009detection,yerramilli2018detection}.
In this context, our methods can be applied to the identification of MPM at different temporal resolutions for frequency-based control strategies. For example, if the predictive model accurately represents low-frequency dynamics but fails to capture high-frequency ones in the plant, MPC controllers can prioritize control actions in the relevant frequency bands to mitigate discrepancies effectively.

\subsection{Climate data analysis}\label{subsec:Climatedata}
 The identification and quantification of the relationships between climate variables in different regions, such as the El Niño-Southern Oscillation (ENSO) and its teleconnections with global weather patterns \cite{wallace1981teleconnections}, is a longstanding important task in climate research.
 In this context, partial correlation analysis has been extensively used \cite{tirabassi2017inferring,rios2015links,zhao2020time}, and it is important to discriminate across temporal resolutions \cite{smith1993long,rodo2006new}. Another example concerns the effects of climate at different time scales on epidemic patterns (cf. \cite{cazelles2023disentangling} and references therein). Thus, our methods represent versatile alternatives for researchers in this field to discriminate partial correlations at the relevant time resolutions.

\subsection{Biomedical signal processing}\label{subsec:Biodata}
Algorithms for learning partial correlation among time series data in biomedical signal processing offer valuable insights into the dynamics of physiological systems and can advance our understanding in several areas, such as brain functions \cite{greicius2003functional,sun2004measuring}, and heart rate variability \cite{li2019spectral}. By analyzing partial correlations across different frequency bands, researchers can unveil complex interactions between physiological signals, identify biomarkers of health and disease, and develop personalized interventions to improve patient outcomes. Additionally, the multi-frequency partial correlation graph learned using our methods can be useful as an input for downstream computational models, for instance, to augment existing brain-fingerprinting approaches \cite{finn2015functional}.
\section{Conclusion}\label{sec:conclusion}
In this work, we have proposed novel methodologies to learn block-sparse multi-frequency PCGs, with the aim of discerning the frequency bands where partial correlation between time series occurs. Specifically, we devise two nonconvex methods, named CF and IA. The former has a closed-form solution and assumes prior knowledge of the number of arcs of the multi-frequency PCG over each layer. The latter jointly learns the CSD matrices and their inverses through an iterative procedure, which does not require any prior knowledge about (block-)sparsity. The IA method efficiently combines SCA and ADMM iterations, providing a robust and highly performing recursive algorithm for learning block-sparse multi-frequency PCGs. Experimental results on synthetic data show that our proposals outperform the state of the art, thus confirming that block-sparsity regularization improves the learning of the $K$-PCG. Interestingly, our proposals are more robust to estimation errors when the number of available samples is scarce or modest, without sacrificing performance in large sample settings. Remarkably, the IA method performs well also with few available samples, even without knowing the measure of sparsity, thanks to its ability to deviate from the error-prone smoothed periodogram.
Finally, we outlined a series of practical applications for which our methods are relevant (but not limited to), across multiple domains. Specifically, the financial case study highlighted that learning a block-sparse multi-frequency PCG reveals valuable information about the partial correlation between time series at various frequency blocks in different market conditions, which is also supported by several economic conjectures of the analyzed periods.

\bibliographystyle{MyIEEE}
\bibliography{bibliography}

\begin{thebibliography}{10}
\providecommand{\url}[1]{#1}
\csname url@samestyle\endcsname
\providecommand{\newblock}{\relax}
\providecommand{\bibinfo}[2]{#2}
\providecommand{\BIBentrySTDinterwordspacing}{\spaceskip=0pt\relax}
\providecommand{\BIBentryALTinterwordstretchfactor}{4}
\providecommand{\BIBentryALTinterwordspacing}{\spaceskip=\fontdimen2\font plus
\BIBentryALTinterwordstretchfactor\fontdimen3\font minus
  \fontdimen4\font\relax}
\providecommand{\BIBforeignlanguage}[2]{{%
\expandafter\ifx\csname l@#1\endcsname\relax
\typeout{** WARNING: IEEEtran.bst: No hyphenation pattern has been}%
\typeout{** loaded for the language `#1'. Using the pattern for}%
\typeout{** the default language instead.}%
\else
\language=\csname l@#1\endcsname
\fi
#2}}
\providecommand{\BIBdecl}{\relax}
\BIBdecl

\bibitem{markowitz1952}
H.~Markowitz, ``Portfolio selection,'' \emph{The Journal of Finance}, vol.~7,
  no.~1, pp. 77--91, 1952.

\bibitem{bullmore2009complex}
E.~Bullmore and O.~Sporns, ``Complex brain networks: {G}raph theoretical
  analysis of structural and functional systems,'' \emph{Nature Reviews
  Neuroscience}, vol.~10, no.~3, pp. 186--198, 2009.

\bibitem{salvador2005undirected}
R.~Salvador, J.~Suckling, C.~Schwarzbauer, and E.~Bullmore, ``Undirected graphs
  of frequency-dependent functional connectivity in whole brain networks,''
  \emph{Philosophical Transactions of the Royal Society B: Biological
  Sciences}, vol. 360, no. 1457, pp. 937--946, 2005.

\bibitem{jacobs2007brain}
J.~Jacobs, M.~J. Kahana, A.~D. Ekstrom, and I.~Fried, ``Brain oscillations
  control timing of single-neuron activity in humans,'' \emph{Journal of
  Neuroscience}, vol.~27, no.~14, pp. 3839--3844, 2007.

\bibitem{ide2017detrended}
J.~Ide, F.~Cappabianco, F.~Faria, and C.-s.~R. Li, ``Detrended partial cross
  correlation for brain connectivity analysis,'' \emph{Advances in Neural
  Information Processing Systems}, vol.~30, 2017.

\bibitem{termenon2016reliability}
M.~Termenon, A.~Jaillard, C.~Delon-Martin, and S.~Achard, ``Reliability of
  graph analysis of resting state f{MRI} using test-retest dataset from the
  {H}uman {C}onnectome {P}roject,'' \emph{NeuroImage}, vol. 142, pp. 172--187,
  2016.

\bibitem{ciuciu2014interplay}
P.~Ciuciu, P.~Abry, and B.~J. He, ``Interplay between functional connectivity
  and scale-free dynamics in intrinsic {fMRI} networks,'' \emph{NeuroImage},
  vol.~95, pp. 248--263, 2014.

\bibitem{ide2018time}
J.~S. Ide and R.~L. Chiang-shan, ``Time scale properties of task and
  resting-state functional connectivity: Detrended partial cross-correlation
  analysis,'' \emph{NeuroImage}, vol. 173, pp. 240--248, 2018.

\bibitem{besserve2010causal}
M.~Besserve, B.~Sch{\"o}lkopf, N.~K. Logothetis, and S.~Panzeri, ``Causal
  relationships between frequency bands of extracellular signals in visual
  cortex revealed by an information theoretic analysis,'' \emph{Journal of
  {C}omputational {N}euroscience}, vol.~29, no.~3, pp. 547--566, 2010.

\bibitem{gu2011precipitation}
G.~Gu and R.~F. Adler, ``Precipitation and temperature variations on the
  interannual time scale: Assessing the impact of {ENSO} and volcanic
  eruptions,'' \emph{Journal of Climate}, vol.~24, no.~9, pp. 2258--2270, 2011.

\bibitem{badwe2009detection}
A.~S. Badwe, R.~D. Gudi, R.~S. Patwardhan, S.~L. Shah, and S.~C. Patwardhan,
  ``Detection of model-plant mismatch in {MPC} applications,'' \emph{Journal of
  Process Control}, vol.~19, no.~8, pp. 1305--1313, 2009.

\bibitem{songsiri2010graphical}
J.~Songsiri, J.~Dahl, and L.~Vandenberghe, ``Graphical models of autoregressive
  processes,'' in \emph{Convex optimization in signal processing and
  communications}.\hskip 1em plus 0.5em minus 0.4em\relax Cambridge
  {U}niversity {P}ress, 2010.

\bibitem{wilson2015models}
G.~T. Wilson, M.~Reale, and J.~Haywood, \emph{Models for dependent time
  series}.\hskip 1em plus 0.5em minus 0.4em\relax CRC Press, 2015, vol. 139.

\bibitem{shrivastava2022methods}
H.~Shrivastava and U.~Chajewska, ``Methods for recovering conditional
  independence graphs: A survey,'' \emph{arXiv preprint arXiv:2211.06829},
  2022.

\bibitem{dahlhaus2000graphical}
R.~Dahlhaus, ``Graphical interaction models for multivariate time series,''
  \emph{Metrika}, vol.~51, no.~2, pp. 157--172, 2000.

\bibitem{banerjee2008model}
O.~Banerjee, L.~El~Ghaoui, and A.~d'Aspremont, ``Model selection through sparse
  maximum likelihood estimation for multivariate {G}aussian or binary data,''
  \emph{The Journal of Machine Learning Research}, vol.~9, pp. 485--516, 2008.

\bibitem{kolar2010}
M.~Kolar, A.~P. Parikh, and E.~P. Xing, ``On sparse nonparametric conditional
  covariance selection,'' in \emph{Proceedings of the 27th International
  Conference on Machine Learning}, ser. ICML'10.\hskip 1em plus 0.5em minus
  0.4em\relax Madison, WI, USA: Omnipress, 2010, p. 559–566.

\bibitem{oh2014optimization}
S.~Oh, O.~Dalal, K.~Khare, and B.~Rajaratnam, ``Optimization methods for sparse
  pseudo-likelihood graphical model selection,'' \emph{Advances in Neural
  Information Processing Systems}, vol.~27, 2014.

\bibitem{belilovsky2017learning}
E.~Belilovsky, K.~Kastner, G.~Varoquaux, and M.~B. Blaschko, ``Learning to
  discover sparse graphical models,'' in \emph{International Conference on
  Machine Learning}.\hskip 1em plus 0.5em minus 0.4em\relax PMLR, 2017, pp.
  440--448.

\bibitem{lugosi2021learning}
G.~Lugosi, J.~Truszkowski, V.~Velona, and P.~Zwiernik, ``Learning partial
  correlation graphs and graphical models by covariance queries.'' \emph{The
  Journal of Machine Learning Research}, vol.~22, pp. 203--1, 2021.

\bibitem{fiecas2011generalized}
M.~Fiecas and H.~Ombao, ``The generalized shrinkage estimator for the analysis
  of functional connectivity of brain signals,'' \emph{The Annals of Applied
  Statistics}, vol.~5, no.~2A, pp. 1102--1125, 2011.

\bibitem{fiecas2014data}
M.~Fiecas and R.~von Sachs, ``Data-driven shrinkage of the spectral density
  matrix of a high-dimensional time series,'' \emph{Electronic Journal of
  Statistics}, vol.~8, no.~2, pp. 2975--3003, 2014.

\bibitem{schneider2016partial}
D.~Schneider-Luftman and A.~T. Walden, ``Partial coherence estimation via
  spectral matrix shrinkage under quadratic loss,'' \emph{IEEE Transactions on
  Signal Processing}, vol.~64, no.~22, pp. 5767--5777, 2016.

\bibitem{bohm2009shrinkage}
H.~B{\"o}hm and R.~von Sachs, ``Shrinkage estimation in the frequency domain of
  multivariate time series,'' \emph{Journal of Multivariate Analysis}, vol.
  100, no.~5, pp. 913--935, 2009.

\bibitem{fiecas2019spectral}
M.~Fiecas, C.~Leng, W.~Liu, and Y.~Yu, ``Spectral analysis of high-dimensional
  time series,'' \emph{Electronic Journal of Statistics}, vol.~13, no.~2, pp.
  4079--4101, 2019.

\bibitem{cai2011constrained}
T.~Cai, W.~Liu, and X.~Luo, ``A constrained $\ell_1$ minimization approach to
  sparse precision matrix estimation,'' \emph{Journal of the American
  Statistical Association}, vol. 106, no. 494, pp. 594--607, 2011.

\bibitem{jung2015graphical}
A.~Jung, G.~Hannak, and N.~Goertz, ``Graphical lasso based model selection for
  time series,'' \emph{IEEE Signal Processing Letters}, vol.~22, no.~10, pp.
  1781--1785, 2015.

\bibitem{foti2016sparse}
N.~J. Foti, R.~Nadkarni, A.~Lee, and E.~B. Fox, ``Sparse plus low-rank
  graphical models of time series for functional connectivity in {MEG},'' in
  \emph{2nd KDD Workshop on Mining and Learning from Time Series}, 2016.

\bibitem{dallakyan2022time}
A.~Dallakyan, R.~Kim, and M.~Pourahmadi, ``Time series graphical lasso and
  sparse {VAR} estimation,'' \emph{Computational Statistics \& Data Analysis},
  vol. 176, p. 107557, 2022.

\bibitem{tugnait2022sparse}
J.~K. Tugnait, ``On sparse high-dimensional graphical model learning for
  dependent time series,'' \emph{Signal Processing}, vol. 197, p. 108539, 2022.

\bibitem{deb2024regularized}
N.~Deb, A.~Kuceyeski, and S.~Basu, ``Regularized estimation of sparse spectral
  precision matrices,'' \emph{arXiv preprint arXiv:2401.11128}, 2024.

\bibitem{whittle1953analysis}
P.~Whittle, ``The analysis of multiple stationary time series,'' \emph{Journal
  of the Royal Statistical Society: Series B (Methodological)}, vol.~15, no.~1,
  pp. 125--139, 1953.

\bibitem{whittle1957curve}
------, ``Curve and periodogram smoothing,'' \emph{Journal of the Royal
  Statistical Society: Series B (Methodological)}, vol.~19, no.~1, pp. 38--47,
  1957.

\bibitem{krampe2024frequency}
J.~Krampe and E.~Paparoditis, ``Frequency domain statistical inference for
  high-dimensional time series,'' \emph{arXiv preprint arXiv:2206.02250}, 2024.

\bibitem{baraniuk2010model}
R.~G. Baraniuk, V.~Cevher, M.~F. Duarte, and C.~Hegde, ``Model-based
  compressive sensing,'' \emph{IEEE Transactions on Information Theory},
  vol.~56, no.~4, pp. 1982--2001, 2010.

\bibitem{scutari2018parallel}
G.~Scutari and Y.~Sun, ``Parallel and distributed successive convex
  approximation methods for big-data optimization,'' in \emph{Multi-agent
  Optimization}.\hskip 1em plus 0.5em minus 0.4em\relax Springer, 2018, pp.
  141--308.

\bibitem{jax2018github}
\BIBentryALTinterwordspacing
J.~Bradbury, R.~Frostig, P.~Hawkins, M.~J. Johnson, C.~Leary, D.~Maclaurin,
  G.~Necula, A.~Paszke, J.~Vander{P}las, S.~Wanderman-{M}ilne, and Q.~Zhang.
  (2018) {JAX}: Composable transformations of {P}ython+{N}um{P}y programs.
\BIBentrySTDinterwordspacing

\bibitem{sardellitti2019graph}
S.~Sardellitti, S.~Barbarossa, and P.~Di~Lorenzo, ``Graph topology inference
  based on sparsifying transform learning,'' \emph{IEEE Transactions on Signal
  Processing}, vol.~67, no.~7, pp. 1712--1727, 2019.

\bibitem{scutari2016parallel}
G.~Scutari, F.~Facchinei, and L.~Lampariello, ``Parallel and distributed
  methods for constrained nonconvex optimization—{P}art {I}: {T}heory,''
  \emph{IEEE Transactions on Signal Processing}, vol.~65, no.~8, pp.
  1929--1944, 2016.

\bibitem{razaviyayn2014parallel}
M.~Razaviyayn, M.~Hong, Z.-Q. Luo, and J.-S. Pang, ``Parallel successive convex
  approximation for nonsmooth nonconvex optimization,'' \emph{Advances in
  Neural Information Processing Systems}, vol.~27, 2014.

\bibitem{fischer2005precoding}
R.~F. Fischer, \emph{Precoding and signal shaping for digital
  transmission}.\hskip 1em plus 0.5em minus 0.4em\relax John Wiley \& Sons,
  2005.

\bibitem{grant2014cvx}
M.~Grant and S.~Boyd, ``{CVX}: Matlab software for disciplined convex
  programming, version 2.1,'' 2014.

\bibitem{diamond2016cvxpy}
S.~Diamond and S.~Boyd, ``{CVXPY}: A python-embedded modeling language for
  convex optimization,'' \emph{The Journal of Machine Learning Research},
  vol.~17, no.~1, pp. 2909--2913, 2016.

\bibitem{boyd2011distributed}
S.~Boyd, N.~Parikh, E.~Chu, B.~Peleato, J.~Eckstein \emph{et~al.},
  ``Distributed optimization and statistical learning via the alternating
  direction method of multipliers,'' \emph{Foundations and
  Trends{\textregistered} in Machine learning}, vol.~3, no.~1, pp. 1--122,
  2011.

\bibitem{parikh2014proximal}
N.~Parikh, S.~Boyd \emph{et~al.}, ``Proximal algorithms,'' \emph{Foundations
  and Trends{\textregistered} in Optimization}, vol.~1, no.~3, pp. 127--239,
  2014.

\bibitem{boyd2004convex}
S.~P. Boyd and L.~Vandenberghe, \emph{Convex optimization}.\hskip 1em plus
  0.5em minus 0.4em\relax Cambridge, UK: Cambridge University Press, 2004.

\bibitem{burden2015numerical}
R.~L. Burden, J.~D. Faires, and A.~M. Burden, \emph{Numerical analysis}.\hskip
  1em plus 0.5em minus 0.4em\relax Cengage Learning, 2015.

\bibitem{d'acunto2023learning}
\BIBentryALTinterwordspacing
G.~D'Acunto, G.~D.~F. Morales, P.~Bajardi, and F.~Bonchi, ``Learning multiscale
  non-stationary causal structures,'' \emph{Transactions on Machine Learning
  Research}, 2023.
\BIBentrySTDinterwordspacing

\bibitem{pearl2009causality}
J.~Pearl, \emph{Causality}.\hskip 1em plus 0.5em minus 0.4em\relax Cambridge
  University Press, 2009.

\bibitem{sun2018large}
Y.~Sun, Y.~Li, A.~Kuceyeski, and S.~Basu, ``Large spectral density matrix
  estimation by thresholding,'' \emph{arXiv preprint arXiv:1812.00532}, 2018.

\bibitem{ombao2001simple}
H.~C. Ombao, J.~A. Raz, R.~L. Strawderman, and R.~Von~Sachs, ``A simple
  generalised crossvalidation method of span selection for periodogram
  smoothing,'' \emph{Biometrika}, vol.~88, no.~4, pp. 1186--1192, 2001.

\bibitem{gharib2021impact}
C.~Gharib, S.~Mefteh-Wali, V.~Serret, and S.~B. Jabeur, ``Impact of {COVID}-19
  pandemic on crude oil prices: {E}vidence from {E}conophysics approach,''
  \emph{Resources Policy}, vol.~74, p. 102392, 2021.

\bibitem{iea2020}
\BIBentryALTinterwordspacing
IEA, ``Gas (2020),'' IEA, Tech. Rep., 2020. [Online]. Available:
  \url{https://www.iea.org/reports/gas-2020}
\BIBentrySTDinterwordspacing

\bibitem{garcia1989model}
C.~E. Garcia, D.~M. Prett, and M.~Morari, ``Model predictive control: {Theory}
  and practice — {A} survey,'' \emph{Automatica}, vol.~25, no.~3, pp.
  335--348, 1989.

\bibitem{kouvaritakis2016model}
B.~Kouvaritakis and M.~Cannon, ``Model predictive control,'' \emph{Switzerland:
  Springer International Publishing}, vol.~38, 2016.

\bibitem{yerramilli2018detection}
S.~Yerramilli and A.~K. Tangirala, ``Detection and diagnosis of model-plant
  mismatch in multivariable model-based control schemes,'' \emph{Journal of
  Process Control}, vol.~66, pp. 84--97, 2018.

\bibitem{wallace1981teleconnections}
J.~M. Wallace and D.~S. Gutzler, ``Teleconnections in the geopotential height
  field during the {N}orthern {H}emisphere winter,'' \emph{Monthly Weather
  Review}, vol. 109, no.~4, pp. 784--812, 1981.

\bibitem{tirabassi2017inferring}
G.~Tirabassi, L.~Sommerlade, and C.~Masoller, ``Inferring directed climatic
  interactions with renormalized partial directed coherence and directed
  partial correlation,'' \emph{Chaos: An Interdisciplinary Journal of Nonlinear
  Science}, vol.~27, no.~3, 2017.

\bibitem{rios2015links}
D.~R{\'\i}os-Cornejo, {\'A}.~Penas, R.~{\'A}lvarez-Esteban, and S.~del
  R{\'\i}o, ``Links between teleconnection patterns and precipitation in
  {S}pain,'' \emph{Atmospheric Research}, vol. 156, pp. 14--28, 2015.

\bibitem{zhao2020time}
J.~Zhao, S.~Huang, Q.~Huang, H.~Wang, G.~Leng, and W.~Fang, ``Time-lagged
  response of vegetation dynamics to climatic and teleconnection factors,''
  \emph{Catena}, vol. 189, p. 104474, 2020.

\bibitem{smith1993long}
R.~L. Smith, V.~Barnett, and K.~Turkman, ``Long-range dependence and global
  warming,'' \emph{Applications of Statistics to Modeling the Earth’s Climate
  System}, pp. 89--92, 1993.

\bibitem{rodo2006new}
X.~Rod{\'o} and M.-{\`A}. Rodr{\'\i}guez-Arias, ``A new method to detect
  transitory signatures and local time/space variability structures in the
  climate system: {T}he scale-dependent correlation analysis,'' \emph{Climate
  Dynamics}, vol.~27, pp. 441--458, 2006.

\bibitem{cazelles2023disentangling}
B.~Cazelles, K.~Cazelles, H.~Tian, M.~Chavez, and M.~Pascual, ``Disentangling
  local and global climate drivers in the population dynamics of mosquito-borne
  infections,'' \emph{Science Advances}, vol.~9, no.~39, p. eadf7202, 2023.

\bibitem{greicius2003functional}
M.~D. Greicius, B.~Krasnow, A.~L. Reiss, and V.~Menon, ``Functional
  connectivity in the resting brain: {A} network analysis of the default mode
  hypothesis,'' \emph{Proceedings of the National Academy of Sciences}, vol.
  100, no.~1, pp. 253--258, 2003.

\bibitem{sun2004measuring}
F.~T. Sun, L.~M. Miller, and M.~D'esposito, ``Measuring interregional
  functional connectivity using coherence and partial coherence analyses of
  {fMRI} data,'' \emph{NeuroImage}, vol.~21, no.~2, pp. 647--658, 2004.

\bibitem{li2019spectral}
K.~Li, H.~R{\"u}diger, and T.~Ziemssen, ``Spectral analysis of heart rate
  variability: {T}ime window matters,'' \emph{Frontiers in Neurology}, vol.~10,
  p. 545, 2019.

\bibitem{finn2015functional}
E.~S. Finn, X.~Shen, D.~Scheinost, M.~D. Rosenberg, J.~Huang, M.~M. Chun,
  X.~Papademetris, and R.~T. Constable, ``Functional connectome fingerprinting:
  {I}dentifying individuals using patterns of brain connectivity,''
  \emph{Nature Neuroscience}, vol.~18, no.~11, pp. 1664--1671, 2015.

\end{thebibliography}

\end{document}